%% file: main.tex
\documentclass[10pt,twocolumn,letterpaper]{article}

\usepackage{cvpr}      

\input{preamble}

\definecolor{cvprblue}{rgb}{0.21,0.49,0.74}
\usepackage[pagebackref,breaklinks,colorlinks,citecolor=cvprblue]
{hyperref}
\usepackage[linesnumbered,ruled,vlined]{algorithm2e}
\usepackage{booktabs,caption}
\usepackage{subcaption}
\usepackage{multirow}
\usepackage{amsmath}
\usepackage{upgreek}
\usepackage{graphicx}
\usepackage{afterpage}
\usepackage{algpseudocode}
\usepackage{amsthm}
\newtheorem{assumption}{Assumption}
\newtheorem{theorem}{Theorem}
\newtheorem{remark}{Remark}
\newtheorem{corollary}{Corollary}
\newtheorem{lemma}{Lemma}
\def\paperID{16835} 
\def\confName{CVPR}
\def\confYear{2024}

\title{DIMAT: Decentralized Iterative Merging-And-Training\\for Deep Learning Models}
\author{Nastaran Saadati$^{1}$ \\
{\tt\small nsaadati@iastate.edu}
\and
Minh Pham$^{2}$ \\
{\tt\small mp5847@nyu.edu}
\and
Nasla Saleem$^{1}$ \\
{\tt\small nasla@iastate.edu}
\and
Joshua R. Waite$^{1}$ \\
{\tt\small jrwaite@iastate.edu}
\and
Aditya Balu$^{1}$ \\
{\tt\small baditya@iastate.edu}
\and
Zhanhong Jiang$^{1}$ \\
{\tt\small zhjiang@iastate.edu}
\and
Chinmay Hegde$^{2}$ \\
{\tt\small chinmay.h@nyu.edu}
\and
Soumik Sarkar$^{1}$ \\
{\tt\small soumiks@iastate.edu}
\and 
$^{1}$Iowa State University, Ames; $^{2}$New York University, New York
\and
 } 

\usepackage[accsupp]{axessibility} 
\begin{document}
\maketitle
\input{sec/0_abstract}    

\input{sec/1_intro}

\input{sec/2_related_work}

\input{sec/3_methodology}

\input{sec/4_mainresults}

\input{sec/5_experimentalresults}
\input{sec/6_conclusion}
{
    \small
    \bibliographystyle{ieeenat_fullname}
    \bibliography{main}
}

\input{sec/X_suppl}

\end{document}

%% file: preamble.tex
\usepackage[dvipsnames]{xcolor}


%% file: sec/0_abstract.tex
\begin{abstract}
Recent advances in decentralized deep learning algorithms have demonstrated cutting-edge performance on various tasks with large pre-trained models. However, a pivotal prerequisite for achieving this level of competitiveness is the significant communication and computation overheads when updating these models, which prohibits the applications of them to real-world scenarios.
To address this issue, drawing inspiration from advanced model merging techniques without requiring additional training, we introduce the Decentralized Iterative Merging-And-Training (DIMAT) paradigm—a novel decentralized deep learning framework. Within DIMAT, each agent is trained on their local data and periodically merged with their neighboring agents using advanced model merging techniques like activation matching until convergence is achieved. DIMAT provably converges with the best available rate for nonconvex functions with various first-order methods, while yielding tighter error bounds compared to the popular existing approaches. We conduct a comprehensive empirical analysis to validate DIMAT's superiority over baselines across diverse computer vision tasks sourced from multiple
datasets.  
 Empirical results validate our theoretical claims by showing that DIMAT attains faster and higher initial gain in accuracy with independent and identically distributed (IID) and non-IID data, incurring lower communication overhead. This DIMAT paradigm presents a new opportunity for the future decentralized learning, enhancing its adaptability to real-world with sparse and light-weight communication and computation.  
\end{abstract}

%% file: sec/1_Intro.tex
\begin{figure}[t]  
  \centering
  \includegraphics[width=1\linewidth]{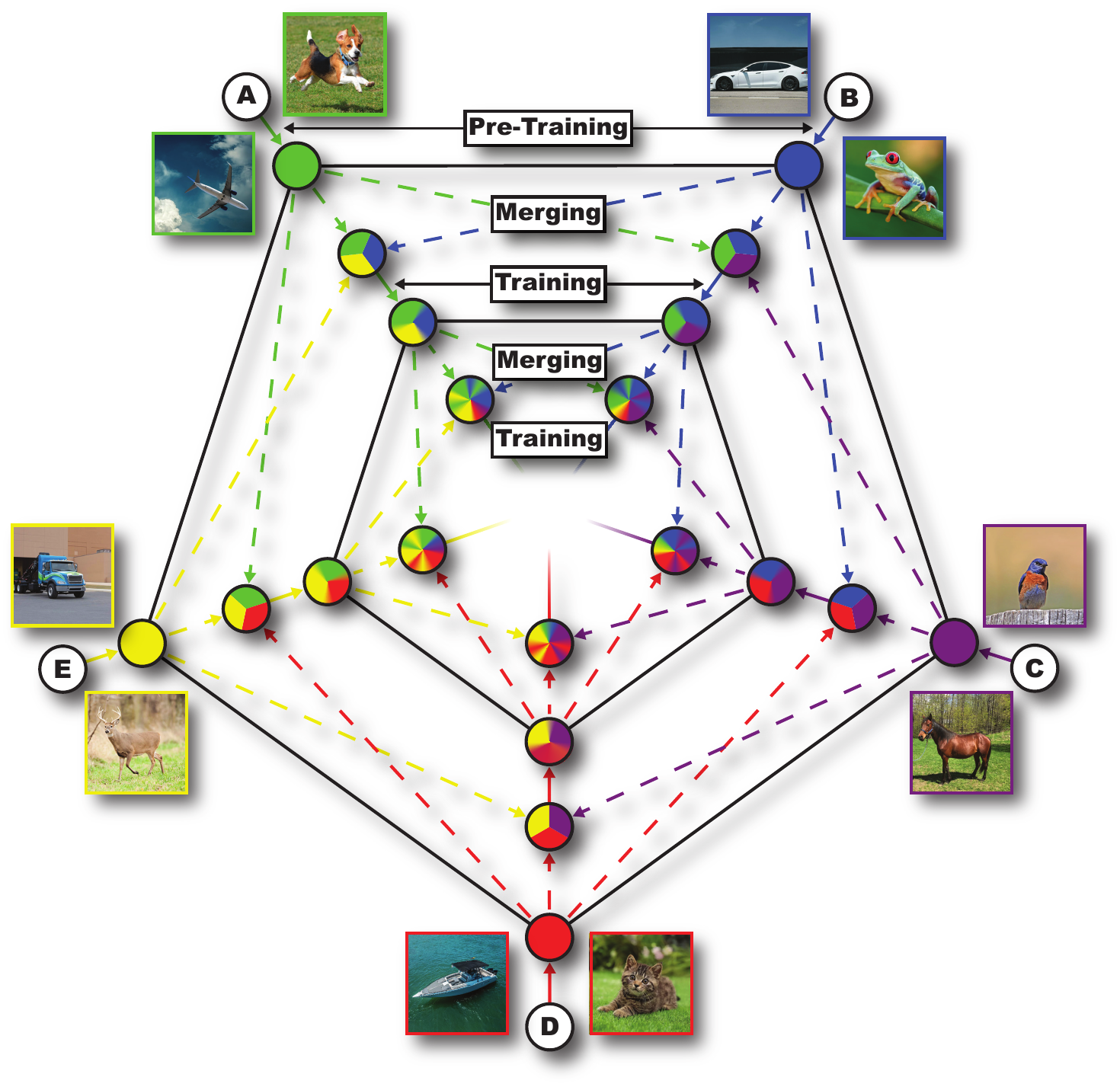}
  \caption{Illustration of \textsc{DIMAT} with ring topology for 5 agents. Agents, denoted by capital letters, undergo pre-training on two unique classes from CIFAR-10 each (solid lines). Subsequently, adjacent agents merge, forming tri-colored nodes (dashed lines). Updated agents then train on the original datasets (solid lines) with a slight increase in dataset colors. This merging and training cycle repeats for specified iterations until final fine-tuning.}
  \label{fig:dimat_schematic}
\end{figure}

\begin{table*}[ht!]
  \centering
  \begin{tabular}{cc}
    \toprule
    Method & Convergence Rate\\
    \midrule
    CD(M)SGD~\cite{yu2019linear} & $\mathcal{O}(\sqrt{\frac{1}{NK}}+\frac{N}{(1-\rho)K}+\frac{N}{(1-\sqrt{\rho})^2K})$ \\
    SGP~\cite{assran2019stochastic}&$\mathcal{O}(\sqrt{\frac{1}{NK}}+\frac{N}{(1-\rho)^2K}+\frac{N}{(1-\rho)^2K^{1.5}})$ \\
    CGA~\cite{esfandiari2021cross}&$\mathcal{O}(\sqrt{\frac{1}{NK}}+\frac{N}{(1-\rho)K}+\frac{\sqrt{N}}{K^{1.5}}+\frac{N}{(1-\sqrt{\rho})^2K})$ \\
    \textsc{DIMAT} & $\mathcal{O}(\sqrt{\frac{1}{NK}}+\frac{N}{(1-\rho')K}+\frac{N}{(1-\sqrt{\rho'})^2K})$ \\
    \bottomrule
  \end{tabular}
  \caption{Comparison between different methods. $N$: \# of agents, $K$: the number of iterations, $0< \rho, \rho'<1$: positive constant related to doubly stochastic matrices and $\rho'\leq \rho$.}
  \label{tab:comparison}
\end{table*}
\section{Introduction}
\label{sec:intro}
Faced with various large-scale deep learning tasks~\cite{qian2023multi,hamdan2023intelligent,beaini2023towards}, researchers and practitioners have made considerable efforts to advance numerous decentralized training algorithms~\cite{beltran2023decentralized}. These also benefit from the rapid development of hardware technologies~\cite{dally2023hardware,rasch2023hardware}. Decentralized deep learning algorithms have demonstrated compellingly cutting-edge performance, nearly matching their centralized counterparts~\cite{sun2021decentralized, esfandiari2021cross, yuan2021decentlam, jiang2021consensus}. The most popular decentralized learning algorithms are first-order methods such as decentralized stochastic gradient descent (SGD)~\cite{jiang2017collaborative,chen2021dacfl,sun2021stability}, their momentum accelerated variants~\cite{balu2021decentralized,yuan2021decentlam,yau2022docom}, and more recently developed decentralized adaptive gradient methods~\cite{chen2023convergence,nazari2022dadam}, which all provably show convergence with a sublinear rate and empirically exhibit appealing performance.

A prerequisite for delivering competitive results is the significant communication and computation overheads, particularly when the models are large such as VGG~\cite{ding2021repvgg} and ResNet~\cite{tai2017image}. However, such an requirement extremely prohibits the use of decentralized learning in real-world scenarios. 
Another concern is the slower performance gain when they are utilized with sparser topology networks, which results in more pronounced communication and computation overheads~\cite{cao2023communication}. Some recent works~\cite{chiu2023laplacian,zhmoginov2023decentralized,le2023refined,wu2023hiflash} have attempted to address the above issues; however, they still remain critically challenging when local agents have widely distinct tasks.

Inspired by a recent line of works in deep learning --- model merging~\cite{pena2023re,stoica2023zipit,csiszarik2023mode,sung2023empirical,ainsworth2022git} --- our paper proposes a novel decentralized deep learning paradigm, that we call \underline{D}ecentralized \underline{I}terative \underline{M}erging-\underline{A}nd-\underline{T}raining (\textsc{DIMAT}). Specifically, within \textsc{DIMAT}, each agent is first trained on their local data and then periodically \textit{merged} with their neighboring agents using advanced model merging techniques. Such a merging-and-training manner is iterated until the convergence is reached. Our method is different from other decentralized deep learning methods significantly driven by vanilla weight averaging~\cite{lian2017can,jiang2017collaborative,chen2021dacfl}; instead, it leverages modulo permutation symmetries from model merging to allow local agents reaching a better ``consensus regime" by enlarging the spectral gap; this eventually leads to a smaller optimality gap. Please see Figure~\ref{fig:dimat_schematic} for the schematic illustration of \textsc{DIMAT}. 

In addition, \textsc{DIMAT} is found to speed up performance gain at the early phase of optimization with lower communication cost. Our contributions are as follows:
\begin{itemize}
    \item We propose and develop \textsc{DIMAT}, a novel decentralized deep learning framework with periodical model merging in the communication protocol. For the local training strategy, \textsc{DIMAT} can be equipped with different first-order methods flexibly. The model merging frequency can even be adjusted in diverse scenarios to reduce the communication overheads.
    \item We theoretically show that \textsc{DIMAT} provably converges with a sublinear rate to a stationary point for nonconvex functions, while yielding a tighter error bound and maintaining linear speed up, compared to the popular existing approaches. The theory also implies the faster initial performance gain due to the larger spectral gap compared to existing algorithms.
    \item A comprehensive empirical analysis is conducted to validate \textsc{DIMAT}'s efficiency and superiority over baselines across IID and non-IID data distributions sourced from three benchmark datasets, including CIFAR-10, CIFAR-100, and Tiny ImageNet datasets by using two popular deep neural network models. See Table~\ref{tab:comparison} for the comparison between the proposed and existing algorithms.
\end{itemize}





%% file: sec/2_related_work.tex
\section{Related Work}
\label{sec:related_work}


\textbf{Decentralized Learning:} Several decentralized learning algorithms have demonstrated performance comparable to centralized counterparts on standard vision datasets. In their comprehensive investigation, Lian et al. \cite{lian2017can} conducted a case study on decentralized parallel stochastic gradient descent (D-PSGD), combining stochastic gradient descent with a gossip averaging algorithm \cite{xiao2004fast}. Jiang et al. introduced consensus-based distributed SGD (CDSGD) exploring collaborative deep learning in fixed topology networks, contributing insights into decentralized cooperation \cite{jiang2017collaborative}. The extension of D-PSGD to directed and time-varying graphs introduced Stochastic Gradient Push (SGP) \cite{assran2019stochastic}, while a momentum version, Decentralized Momentum Stochastic Gradient Descent (DMSGD), was proposed in \cite{balu2021decentralized}. Fotouhi et al. proposed an algorithm called minimum connected Dominating Set Model Aggregation (DSMA) to address communication overhead issues \cite{FOTOUHI202425}. However, a critical assumption for achieving state-of-the-art performance in these decentralized algorithms is that the data is IID across agents \cite{assran2019stochastic, jiang2017collaborative, FOTOUHI202425}. Although, Khawatmi et al. delved into decentralized clustering and linking, shedding light on challenges posed by heterogeneous settings and node assignment within clusters \cite{khawatmi2017decentralized}. Recent efforts aim to bridge the performance gap between IID and non-IID data in decentralized setups \cite{tang2018d, pu2021distributed, esfandiari2021cross, aketi2022neighborhood, vogels2021relaysum}. Lin et al. proposed a  Quasi-Global momentum-based decentralized deep learning approach for heterogeneous data \cite{lin2021quasi}. Tang et al. introduced the \(D^2\) algorithm, an extension of D-PSGD tailored for non-IID data, while Nadiradze et al. proposed SwarmSGD, utilizing random interactions between agents in a graph for consensus \cite{tang2018d, nadiradze2021asynchronous}. Esfandiari et al. introduced Cross-Gradient Aggregation (CGA) and its compressed variant (CompCGA) for decentralized learning on entirely non-IID data, asserting superior performance \cite{esfandiari2021cross}. However, these techniques incur higher communication costs compared to standard decentralized algorithms such as DSGD. Federated learning methods, such as the one proposed by McMahan et al. \cite{mcmahan2017communication}, offer a practical solution for training deep networks by iteratively averaging models, thus reducing communication costs compared to synchronous stochastic gradient descent. Federated Multi-Task Learning \cite{smith2017federated}, demonstrated the natural suitability of multi-task learning to handle the statistical challenges of federated learning settings. Our primary aim is to leverage multi-task learning and decentralized learning on IID and non-IID data with minimal computational and communication overhead.

\textbf{Model Merging:} Recently, model merging has emerged as a technique that integrates the parameters or predictions of multiple deep-learning models into a single one, requiring minimal computation. This process exploits the proximity of models with the same pre-trained weights in the same error basin \cite{neyshabur2020being}. Numerous studies \cite{huang2017snapshot, izmailov2018averaging, von2020neural, wortsman2022robust} have utilized this characteristic to average the weights of a model during training. The merging process can enhance single-task performance for models designed for the same task \cite{wortsman2022model} or create a multi-task solution for models targeting different tasks \cite{ilharco2022editing, wortsman2022robust, don2022cold}. Various merging methods such as linear interpolation, and task arithmetic \cite{ilharco2022editing, jin2022dataless} have also been proposed. However, for models without shared pre-trained weights, weight permutation may be necessary before merging \cite{entezari2021role, ainsworth2022git, jordan2022repair}. In our approach, we leverage modulo permutation symmetries from model merging, enabling local agents to reach a more robust ``consensus regime". This integrated approach aims to address challenges arising from heterogeneous data distributions, providing an effective solution with minimal computational and communication overhead.





%% file: sec/3_methodology.tex
\section{Methodology}
\label{sec:methodology}
In this section, we first present some preliminaries regarding one model merging technique called activation matching, and then formulate the generic optimization problem in the decentralized learning setting, followed by the proposed algorithms.
\subsection{Preliminaries: Activation Matching}
In this section, we present the activation matching methodologies, extending the principles established by Ainsworth et al.~\cite{ainsworth2022git}. These methodologies form the cornerstone of our decentralized model merging approach, inspired by the concept that models acquiring similar features are likely to perform similar tasks~\cite{li2016convergent}.
Consider two multi-layer perceptron models, \(M_1\) and \(M_2\). Given activations for each model, we aim to associate each unit in \(M_1\) with a unit in \(M_2\). This is accomplished by fitting the linear relationship into the regression framework, constraining a matrix of cross-correlation coefficients to solutions in the set of all symmetric permutation matrices, $S_d\in\mathbb{R}^{d\times d}$.
For the \(\ell\)-th layer, let \(\mathbf{Z}^{(M_1)} \in \mathbb{R}^{d \times s}\) and \(\mathbf{Z}^{(M_2)} \in \mathbb{R}^{d \times s}\) denote the cross-correlation of the activations of models \(M_1\) and \(M_2\), respectively, where $s$ signifies the number of all training data points in models. We aim to find a permutation matrix \(\mathbf{P}_\ell\in\mathbb{R}^{d\times d}\) that minimizes the Frobenius norm of the difference between \(\mathbf{Z}^{(M_1)}\) and the permuted \(\mathbf{Z}^{(M_2)}\):
\begin{equation}
    \begin{split}
        \mathbf{P}_\ell &= \arg\min_{\mathbf{P} \in S_d} \sum_{p=1}^{s} \|\mathbf{Z}^{(M_1)}_{:,p} - \mathbf{P} \mathbf{Z}^{(M_2)}_{:,p}\|^2 \\
        &= \arg\max_{\mathbf{P} \in S_d} \langle \mathbf{P}, \mathbf{Z}^{(M_1)}(\mathbf{Z}^{(M_2)})^\top \rangle_F.
    \end{split}
\end{equation}
Please note that the dimension of $\mathbf{P}$ will be expanded in the decentralized learning setting in the rest of the paper.
This problem constitutes a ``linear assignment problem" (LAP), for which efficient algorithms exist. Once LAP is solved for each layer, we proceed to permute the parameters of model \(M_2\) to align them with those of model \(M_1\):
\[
\mathbf{\mathcal{W}}_\ell' = \mathbf{P}_\ell \mathbf{\mathcal{W}}_\ell^{M_2} \mathbf{P}_{\ell-1}^\top, \quad b_\ell' = \mathbf{P}_\ell b_\ell^{M_2},
\] where $\mathbf{\mathcal{W}}'_\ell$ and $\mathbf{\mathcal{W}}^{M_2}_\ell$are weight matrices, $b'_\ell$ and $b_\ell^{M_2}$ are bias vectors.
We will utilize the activation matching method in the communication phase of our DIMAT framework for decentralized model merging. This approach represents a more sophisticated strategy compared to the basic averaging technique employed in prior methods, such as consensus-based decentralized SGD (CDSGD) or its momentum variant~\cite{jiang2017collaborative, yu2019linear}. Throughout the paper, we still use model merging instead of activation matching to better fit the DIMAT name. More detailed information about this method has been provided in the  Supplementary Materials.



\subsection{Problem Formulation}
Consider a network involving $N$ agents and denote by $\mathcal{G}=(\mathcal{V}, \mathcal{E})$ the connected topology, where $\mathcal{V} = \{1,2,...,N\}$ and $\mathcal{E}\subseteq \mathcal{V}\times \mathcal{V}$. If $(i, j)\in\mathcal{E}$, then agent $i$ is able to communicate with agent $j$. We also define the neighborhood of agent $i$ as follows: $Nb(i):=\{j\in\mathcal{V}:(i,j)\in\mathcal{E}\;or\;j=i\}$. Without loss of generality, we assume the graph $\mathcal{G}$ is {connected} and {undirected}. The $N$ agents jointly solve the following consensus optimization problem:
\begin{equation}\label{opt_problem}
    \text{min}_{\mathbf{x}\in\mathbb{R}^d} f(\mathbf{x}) = \frac{1}{N}\sum_{i=1}^N\mathbb{E}_{\xi_i\sim\mathcal{D}_i}[F^i(\mathbf{x};\xi_i)],
\end{equation}
where $f^i(\mathbf{x}):=\mathbb{E}_{\xi_i\sim\mathcal{D}_i}[F^i(\mathbf{x};\xi_i)]$ are smooth non-convex functions with different data distributions $\mathcal{D}_i$. We denote by $\mathbf{g}^i$ the mutually independent unbiased stochastic gradients sampled at points $\mathbf{x}^i\in\mathbb{R}^d$ such that $\mathbb{E}[\mathbf{g}^i]=\nabla f^i(\mathbf{x}^i)$.

In decentralized learning, consensus averaging is a critical key to maintain {closeness} among agents that learn to achieve the shared model parameter $\mathbf{x}^*$. Such a mechanism works well when the sampled datasets for individual agents are independent and identically distributed (IID). However, in reality, IID data is rare, and data heterogeneity needs to be considered, particularly when agents implement diverse tasks. To address this issue, numerous works have attempted to develop decentralized learning algorithms that are equipped with more complex communication protocols \cite{assran2019stochastic, jiang2017collaborative, esfandiari2021cross}. Nevertheless, they essentially fall into variants of consensus averaging and require quantification techniques to alleviate the significant communication overhead.

In this study, we turn our direction to one recently developed model merging technique that has empirically been studied for multi-task learning~\cite{ainsworth2022git}. As introduced above, the core is to apply a suitable {permutation} to the weight matrices such that the new model can be adapted to unseen tasks without additional training from scratch. 
Thus, we extend model merging to the decentralized learning setting for the first time and mathematically show the convergence rate. 
We next introduce one operator that fits in the parameter space upon which vectors are defined to represent models. 

Define the \textit{stochastic model merging} operator as $\mathcal{T}(\cdot|\mathbf{\Pi}):\mathbb{R}^d\times\mathbb{R}^d\times...\times\mathbb{R}^d\to\mathbb{R}^d$ and the permutation matrix as $\mathbf{P}^{ij}\in\mathbb{R}^{d\times d}$ respectively, where $\mathbf{\Pi}\in\mathbb{R}^{N\times N}$ is the mixing matrix. Then we have the following:
\begin{equation}
    \mathcal{T}(\mathbf{x}^1,...,\mathbf{x}^q|\mathbf{\Pi}) = \sum_{j\in Nb(i)}\pi_{ij}\mathbf{P}^{ij}\mathbf{x}^j,
\end{equation}
where $q:=|Nb(i)|\leq N$, $\pi_{ij}$ is the element in $\mathbf{\Pi}$ at $i$-th row and $j$-th column.
It can be observed that vanilla averaging is the simplest model merging, resulting in $\mathbf{P}^{ij}=\mathbf{I}$. $\mathbf{P}^{ij}$ will be time-varying along with the update of $\mathbf{x}$ and this intuitively makes sense, as for each iteration, the model merging may apply to different features on the layers. Further details on the mixing matrix definition are provided in the Supplementary Materials.

\subsection{Algorithmic Framework}
Decentralized learning typically comprises two crucial steps, communication and computation. The computation step corresponds to local model update that can possibly be accomplished by first-order methods. In this context, we combine the proposed DIMAT with SGD, momentum SGD (MSGD), and Adam~\cite{zhang2018improved} to develop \textsc{DIMAT-SGD}, \textsc{DIMAT-MSGD} and \textsc{DIMAT-Adam}.
\begin{algorithm}
  \caption{\textsc{DIMAT-SGD}}
  \label{alg:dmm_sgd}
  \SetKwInOut{Input}{Input}
  \SetKwInOut{Output}{Output}
  \Input{mixing matrix $\mathbf{\Pi}$, the \# of iterations $K$, initialization $\mathbf{x}_1^i, \forall i\in\mathcal{V}$, step size $\alpha$, merging frequency $n$}
  \Output{$\bar{\mathbf{x}}_K=\frac{1}{N}\sum_{i=1}^N\mathbf{x}_K^i$}  
  \BlankLine
  \For{ $k$ in $1:K$ }
  { 
    \For{each agent $i\in\mathcal{V}$}
    { 
    Calculate the stochastic gradient $\mathbf{g}^i_k$\;
    \eIf{$k$ mod $n$=0}
    { $\mathbf{x}_{k+1/2}^i=\sum_{j\in Nb(i)}\pi_{ij}\mathbf{P}_k^{ij}\mathbf{x}^j_{k+1/2}$\;
    }{$\mathbf{x}^i_{k+1/2}=\mathbf{x}^i_{k}$\;}$\mathbf{x}_{k+1}^i=\mathbf{x}_{k+1/2}^i-\alpha \mathbf{g}^i_k$\;
    }
  }
\end{algorithm}
In Algorithm~\ref{alg:dmm_sgd}, Line 5 implies that the frequency of merging step can be implemented periodically, which reduces the number of communication rounds. The term $k + 1/2$ denotes the update for consensus. The permutation matrix $\mathbf{P}_k^{ij}$ in model merging between two different models is essentially obtained through the activation matching~\cite{ainsworth2022git}. Regardless of the specific detail of how permutation is completed among models, $\mathbf{P}_k^{ij}$ always remains a \textit{doubly stochastic} matrix. Thus, to analyze the convergence property of \textsc{DIMAT-SGD}, we are able to couple the mixing matrix $\mathbf{\Pi}$ and permutation matrix $\mathbf{P}_k^{ij}$ in a higher dimension involving multiple agents. Analogously, when the computation step employs MSGD and Adam, the Algorithms~\ref{alg:dmm_msgd} and~\ref{alg:dmm_adam} are attained accordingly. Please see the Supplementary Materials for these two algorithms and their associated analysis in more detail. 

%% file: sec/4_mainresults.tex
\section{Main Results}
\label{sec:main_results}

\subsection{Assumptions}
Before presenting the main results, we first state several assumptions to serve the analysis. Throughout the rest of the analysis, the following standard assumptions hold true. We also defer all proof to the Supplementary Materials.
\begin{assumption}\label{assumption_1}
Problem~\ref{opt_problem} satisfies the following:
\begin{itemize}
    \item \textbf{Smoothness.} Each function $f^i(\mathbf{x})$ is smooth with modulus $L$.
    \item \textbf{Bounded variances.} There exist $\sigma, \kappa > 0$ such that
    \begin{equation}
        \mathbb{E}_{\xi\sim\mathcal{D}_i}[\|\nabla F^i(\mathbf{x};\xi)-\nabla f^i(\mathbf{x})\|^2]\leq \sigma^2, \forall i,\forall \mathbf{x}.
    \end{equation}
    \begin{equation}
        \frac{1}{N}\sum_{i=1}^N\|\nabla f^i(\mathbf{x})-\nabla f(\mathbf{x})\|^2\leq \kappa^2, \forall \mathbf{x}.
    \end{equation}
\end{itemize}
\end{assumption}
The smoothness in Assumption~\ref{assumption_1} is quite generic in decentralized learning algorithms~\cite{yu2019linear,jiang2017collaborative, esfandiari2021cross} as it provides the guarantee of loss descent for the analysis. $\sigma^2$ signifies the upper bound of variances of stochastic gradients for local agents, while $\kappa^2$ quantifies the gradient diversity between each agent's local objective loss $f^i(\mathbf{x})$, due to the different data distributions. One can also use the bounded second moment of stochastic gradients assumption~\cite{stich2018local}, which is stronger and results in a looser error bound.
\begin{assumption}\label{assumption_2}
   The mixing matrix $\mathbf{\Pi}\in\mathbb{R}^{N\times N}$ is a symmetric doubly stochastic matrix satisfying $\lambda_1(\mathbf{\Pi})=1$ and
   \begin{equation}
       \textnormal{max}\{|\lambda_2(\mathbf{\Pi})|, |\lambda_N(\mathbf{\Pi})|\}\leq \sqrt{\rho}<1,
   \end{equation}
where $0<\rho<1$, $\lambda_l(\cdot)$ is the $l$-th largest eigenvalue of the matrix.
\end{assumption}
The assumption for the mixing matrix $\mathbf{\Pi}$ has been utilized frequently in existing works~\cite{esfandiari2021cross,yu2019linear}. In our analysis, to consider multiple agents, we will have the expanded mixing matrix of the Kronecker product between $\mathbf{\Pi}$ and $\mathbf{I}_d$, $\mathbf{W}=\mathbf{\Pi}\otimes \mathbf{I}_d$, but the magnitudes of eigenvalues of $\mathbf{W}$ remain the same through a known result presented in the sequel and the fact that eigenvalue of $\mathbf{I}_d$ is 1.
\begin{theorem}\cite{schacke2004kronecker}\label{kronecker_prod}
    Let $\mathbf{C}\in\mathbb{R}^{N\times N}$ and $\mathbf{D}\in\mathbb{R}^{d\times d}$, with eigenvalue $\lambda\in s(\mathbf{C})$ with corresponding eigenvector $x\in\mathbb{C}^{N}$, and $\mu\in s(\mathbf{D})$ with corresponding eigenvector $y\in\mathbb{C}^{d}$, where $s(\cdot)$ signifies the spectrum of a matrix. Then $\lambda\mu$ is an eigenvalue of $\mathbf{C}\otimes \mathbf{D}$ with corresponding eigenvector $x\otimes y\in\mathbb{C}^{dN}$. Any eigenvalue of $\mathbf{C}\otimes \mathbf{D}$ arises as such a product of eigenvalues of $\mathbf{C}$ and $\mathbf{D}$.
\end{theorem}
One immediate outcome from Theorem~\ref{kronecker_prod} is that $\textnormal{max}\{|\lambda_2(\mathbf{W})|, |\lambda_{dN}(\mathbf{W})|\}\leq \sqrt{\rho}<1$.
Define the permutation matrix for all agents at the $k$-th time step as $\mathbf{P}_k\in\mathbb{R}^{dN\times dN}$. Based on the algorithmic frameworks, it is known that in the analysis we will have to deal with the consecutive matrix product in a form of $\prod_{\tau=1}^k\mathbf{W}\mathbf{P}_\tau$, which is the non-trivial part and distinguished from the existing analysis. 
However, $\mathbf{P}_\tau$ is a symmetric doubly stochastic matrix such that $\mathbf{W}\mathbf{P}_\tau$ is a doubly stochastic matrix. Following similarly Assumption~\ref{assumption_2}, we set another assumption for $\mathbf{W}\mathbf{P}_k, \forall k\geq 1$ in the sequel.
\begin{assumption}\label{assumption_3}
   The matrix $\mathbf{W}\mathbf{P}_k\in\mathbb{R}^{dN\times dN}, \forall k\geq 1$ is a symmetric doubly stochastic matrix satisfying $\lambda_1(\mathbf{W}\mathbf{P}_k)=1$ and
   \begin{equation}
       \textnormal{max}\{|\lambda_2(\mathbf{W}\mathbf{P}_k)|, |\lambda_{dN}(\mathbf{W}\mathbf{P}_k)|\}\leq \sqrt{\rho'}<1,
   \end{equation}
where $0<\rho'<1$, $\lambda_l(\cdot)$ is the $l$-th largest eigenvalue of the matrix.
\end{assumption}
The above assumption is mild as it only extends the similar conclusion in Assumption~\ref{assumption_2} to the matrix product. Assumption~\ref{assumption_3} not only characterizes the convergence for \textsc{DIMAT}, but also providing a justification of $\rho'\leq\rho$, which implies the tighter error bounds presented in the sequel. This holds due to the relationship between singular value and eigenvalue, symmetric properties, the fact that eigenvalues of a permutation matrix lie on the unit circle, and Courant–Fischer–Weyl Min-Max Principle~\cite{il2022min}. The detailed analysis is deferred to the Supplementary Materials.
\subsection{Convergence Analysis}
Throughout the analysis, we define $f^*:=f(\mathbf{x}^*)>-\infty$, where $\mathbf{x}^*=\textnormal{argmin}_{\mathbf{x}\in\mathbb{R}^d}f(\mathbf{x})$. We also set $n=1$ in the analysis.
\begin{theorem}\label{dmm-sgd-theo}
    Let Assumptions~\ref{assumption_1} and~\ref{assumption_3} hold. If the step size $\alpha\leq \textnormal{min }\{\frac{1-\sqrt{\rho'}}{4\sqrt{2}L}, \frac{\sqrt{(1-\sqrt{\rho'})^4+64(1-\sqrt{\rho'})^2}-(1-\sqrt{\rho'})^2}{32L}\}$ in Algorithm~\ref{alg:dmm_sgd}, then for all $K\geq 1$, the following relationship holds true:
    \begin{equation}
    \begin{split}
        \frac{1}{K}\sum_{k=1}^K\mathbb{E}[\|\nabla f(\bar{\mathbf{x}}_k)\|^2]&\leq \frac{2(f(\bar{\mathbf{x}}_0)-f^*)}{\alpha K}+\frac{4\alpha^2\sigma^2L^2}{1-\rho'}\\&+\frac{16\alpha^2\kappa^2L^2}{(1-\sqrt{\rho'})^2}+\frac{L\alpha\sigma^2}{N},
    \end{split}
    \end{equation}
where $\bar{\mathbf{x}}_k=\frac{1}{N}\sum_{i=1}^N\mathbf{x}^i_k$.
\end{theorem}
\begin{remark}\label{remark_1}
    $\mathbb{E}[\|\nabla f(\bar{\mathbf{x}}_k)\|^2]$ has generically been applied as a metric to evaluate the convergence behavior in decentralized learning algorithms. In this study, the usage of stochastic model merging should have induced another metric different from $\mathbb{E}[\|\nabla f(\bar{\mathbf{x}}_k)\|^2]$. However, we still adopt it to assess the convergence rate since it only acts at the end of each iteration for all models. Within each iteration, the vanilla weight averaging has been replaced with the model merging.
\end{remark}
\begin{corollary}\label{dmm-sgd-coro}
        Let Assumptions~\ref{assumption_1} and~\ref{assumption_3} hold. If step size $\alpha=\mathcal{O}(\sqrt{\frac{N}{K}})$ in Algorithm~\ref{alg:dmm_sgd}, then for all $K\geq\textnormal{max}\{\frac{32NL^2}{\upzeta},\frac{512NL^2}{\upzeta^2+32\upzeta-\upzeta\sqrt{\upzeta^2+64\upzeta}}\}$, where $\upzeta=(1-\sqrt{\rho'})^2$, we have
        \begin{equation}
        \begin{split}
            \frac{1}{K}\sum_{k=1}^K\mathbb{E}[\|\nabla f(\bar{\mathbf{x}}_k)\|^2]&\leq \mathcal{O}(\sqrt{\frac{1}{NK}}+\frac{N}{(1-\rho')K}\\&+\frac{N}{(1-\sqrt{\rho'})^2K}).
        \end{split}
        \end{equation}
\end{corollary}
\begin{remark}\label{remark_2}
    Corollary~\ref{dmm-sgd-coro} implies that \textsc{DIMAT-SGD} achieves the best available convergence rate but yields a tighter error bound due to $\rho'\leq\rho$, which suggests the impact of model merging on the topology. Before the optimization enters into the regime where $\mathcal{O}(\sqrt{\frac{1}{NK}
    })$ has dominated, \textsc{DIMAT-SGD} reduces the optimization error, as $1-\rho'$ is larger. This intuitively tells us that the \textsc{DIMAT} framework leads to the faster initial performance gain with less communication and computation cost.
    On the other hand, when $K$ is sufficiently large such that $\mathcal{O}(\sqrt{\frac{1}{NK}
    })$ dominates the convergence, the linear speed up due to $N$ is achieved accordingly. In this regime, the impact of topology is much smaller, which indicates that \textsc{DIMAT-SGD} and \textsc{CDSGD}~\cite{jiang2017collaborative,yu2019linear} will have ultimately similar performance. Our empirical results will evidently support the theoretical findings. The similar theoretical implications also apply to \textsc{DIMAT-MSGD} and \textsc{DIMAT-Adam}, given additional analysis presented in the Supplementary Materials.
\end{remark}

%% file: sec/5_experimentalresults.tex
\begin{figure*}[ht!]
    \centering

    \begin{subfigure}[b]{0.32\textwidth}
        \centering
        \includegraphics[width=\linewidth]{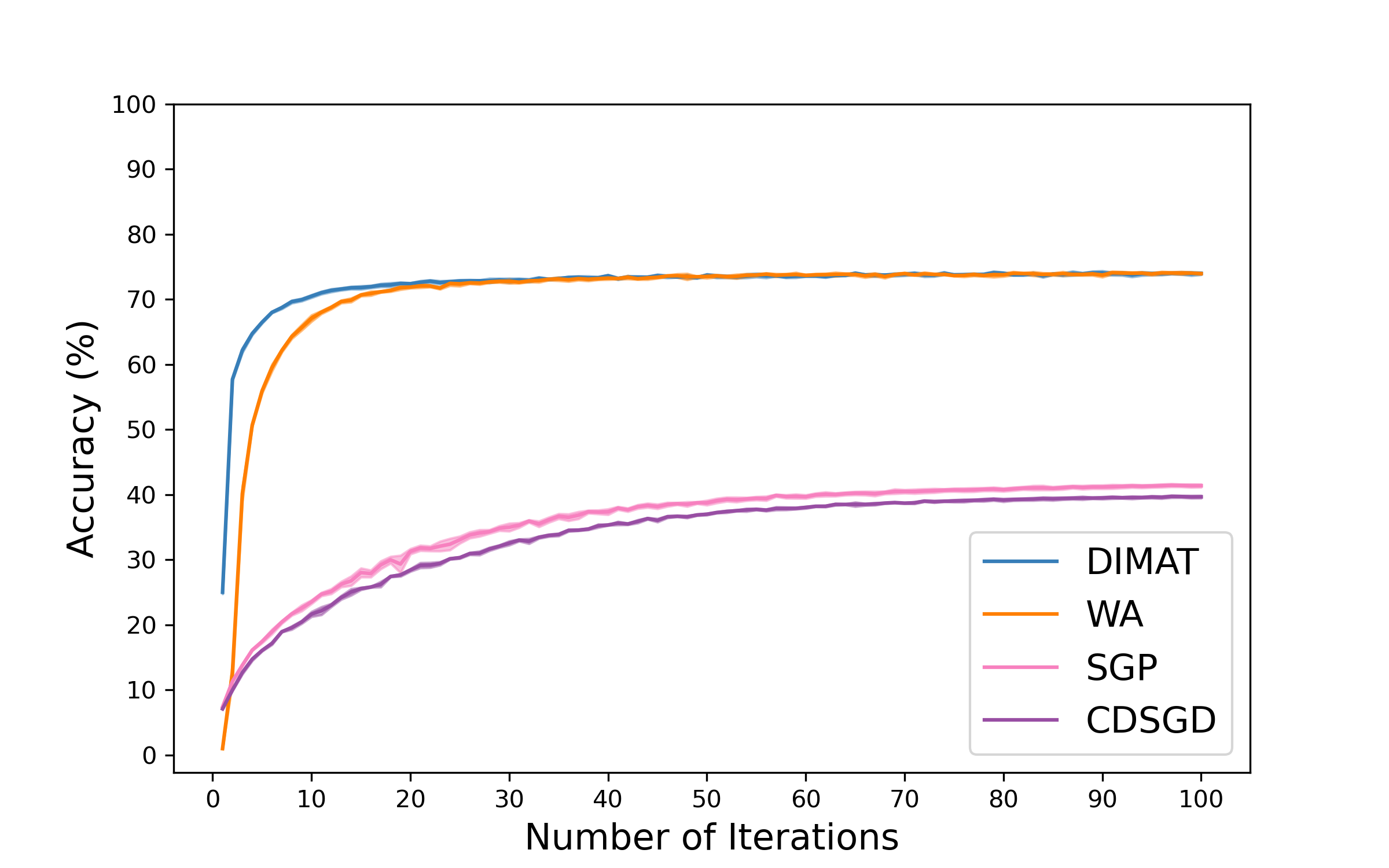}
        \caption{FC Topology for 5 Agents}
        \label{fig:fig:RN25IIDFC}
    \end{subfigure}
    \begin{subfigure}[b]{0.32\textwidth}
        \centering
        \includegraphics[width=\linewidth]{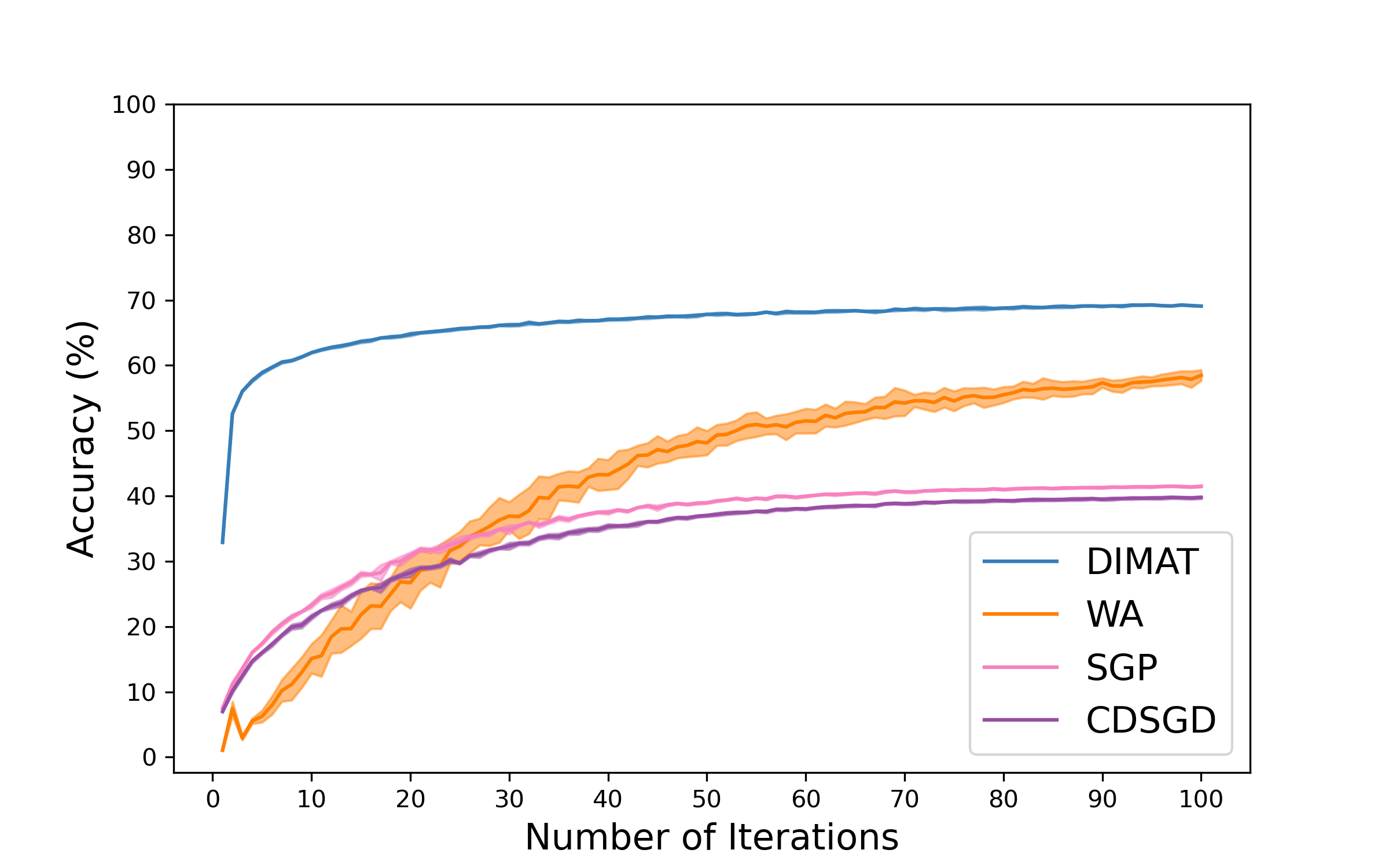}
        \caption{Ring Topology for 5 Agents}
        \label{fig:RN25IIDR}
    \end{subfigure}
    \begin{subfigure}[b]{0.32\textwidth}
        \centering
        \includegraphics[width=\linewidth]{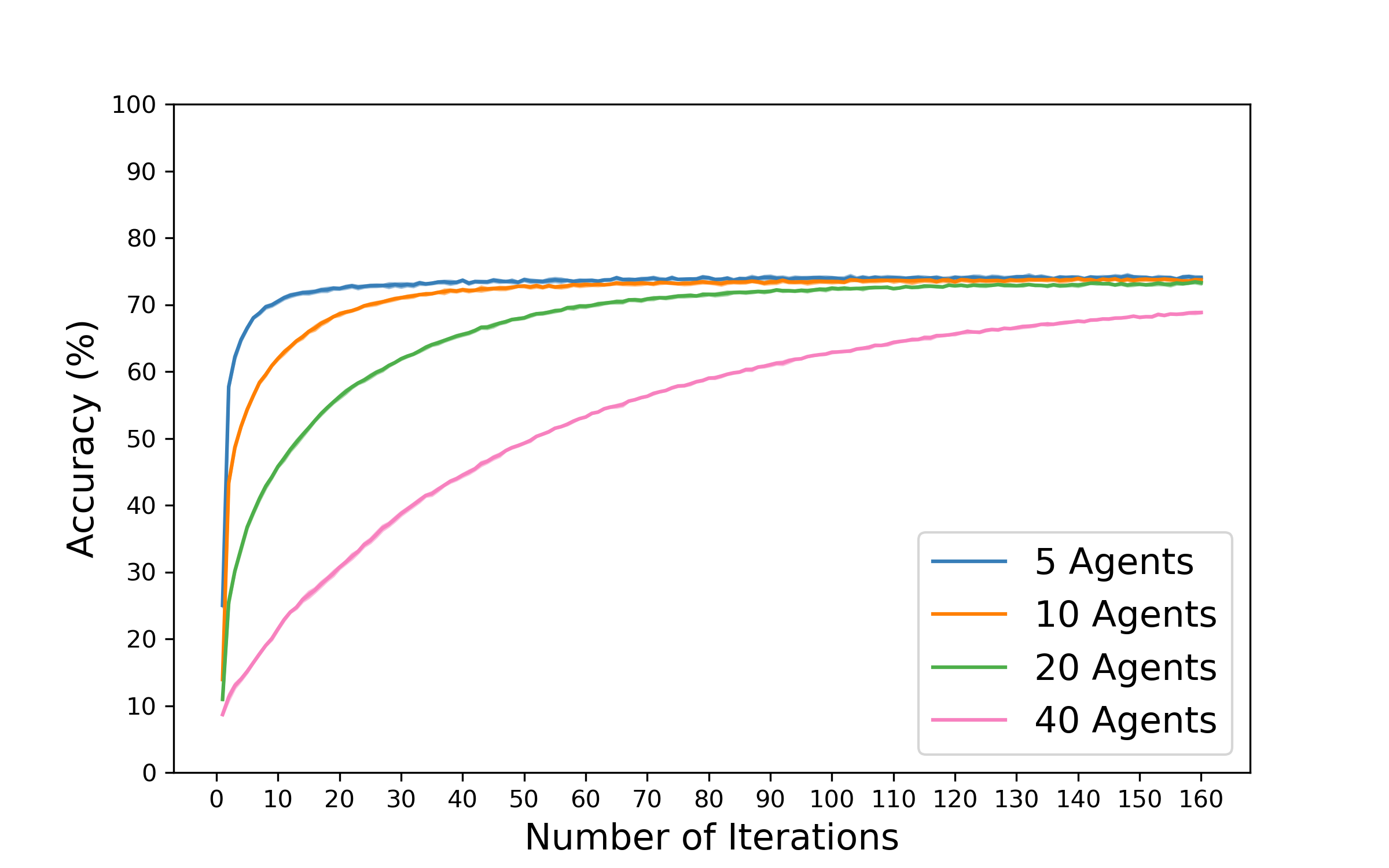}
        \caption{Scalability Analysis for FC Topology}
        \label{fig:RNIIDsc}
    \end{subfigure}
    
    \caption{Comparing algorithmic accuracy (mean$\pm$std) in fully connected (FC) (a) and ring (b) topologies with ResNet-20 architecture on CIFAR-100 IID data. The scalability with ResNet-20 architecture on CIFAR-100 IID data and fully connected topology is shown in (c).}
    \label{fig:RNIIDC}
\end{figure*}
\begin{figure*}[ht!]
    \centering

    \begin{subfigure}[b]{0.33\textwidth}
        \centering
        \includegraphics[width=\linewidth]{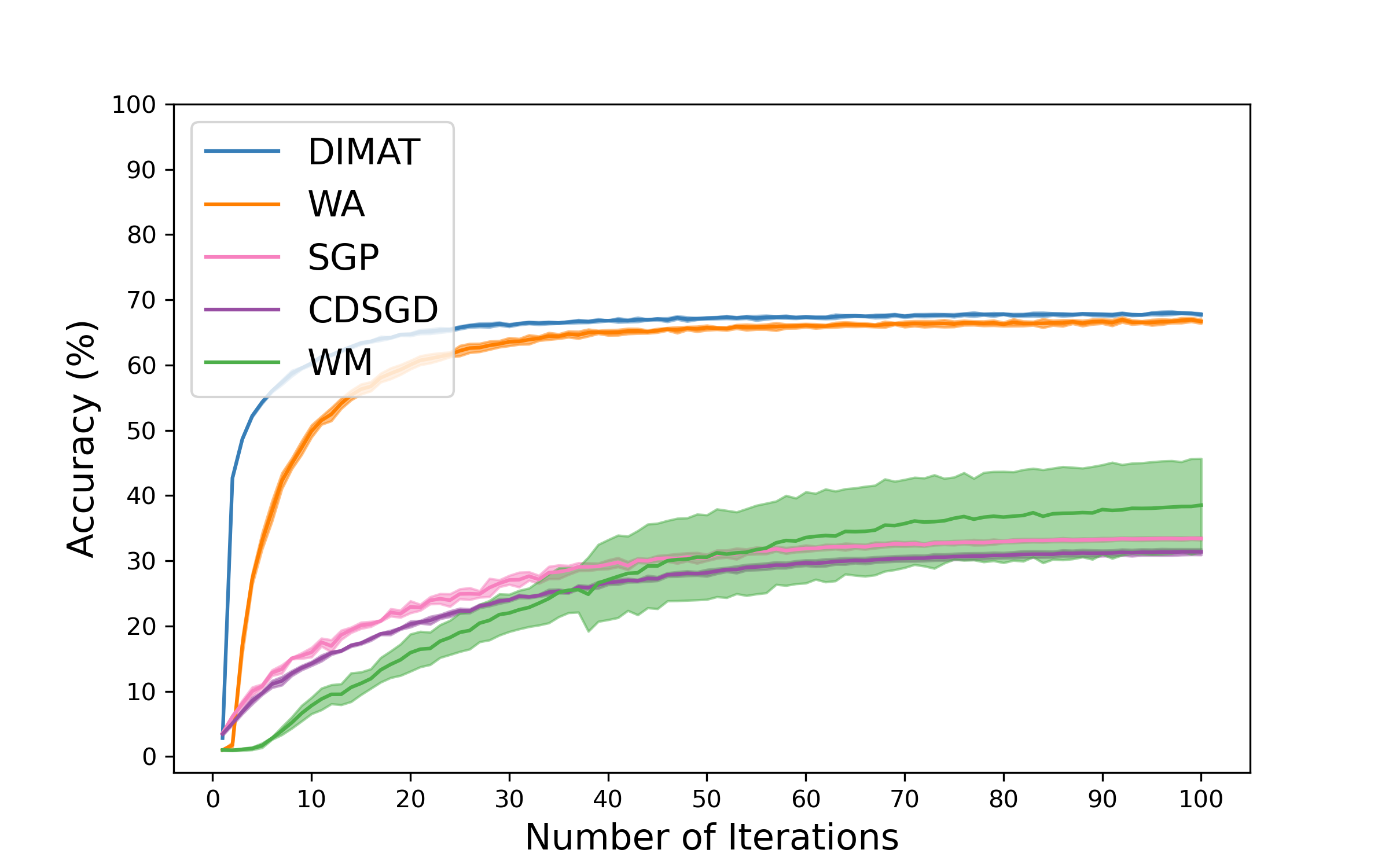}
        \caption{FC Topology for 5 Agents}
        \label{fig:fig:VV25IIDFC}
    \end{subfigure}
    \begin{subfigure}[b]{0.33\textwidth}
        \centering
        \includegraphics[width=\linewidth]{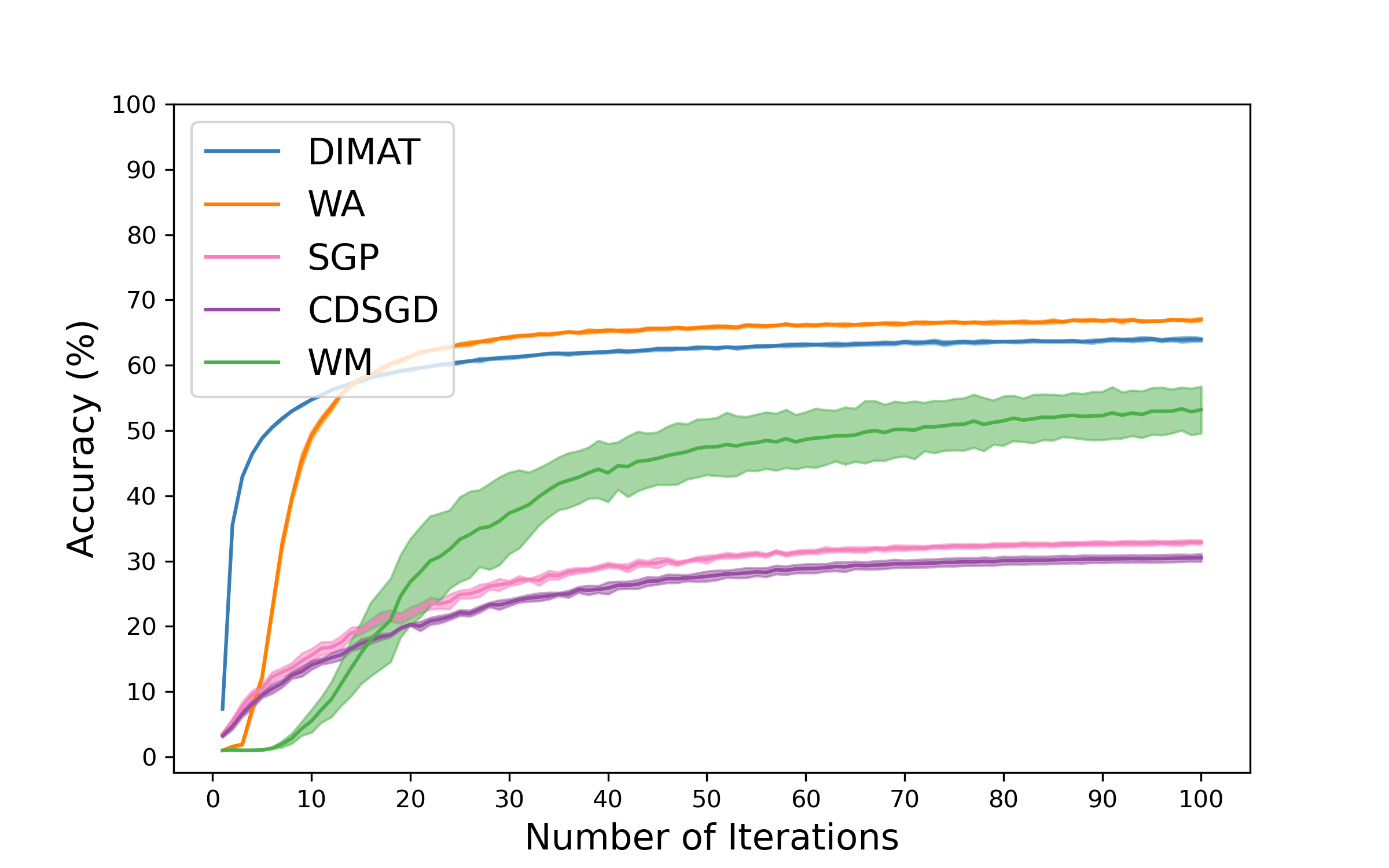}
        \caption{Ring Topology for 5 Agents}
        \label{fig:VV25IIDR}
    \end{subfigure}
    \begin{subfigure}[b]{0.33\textwidth}
        \centering
        \includegraphics[width=\linewidth]{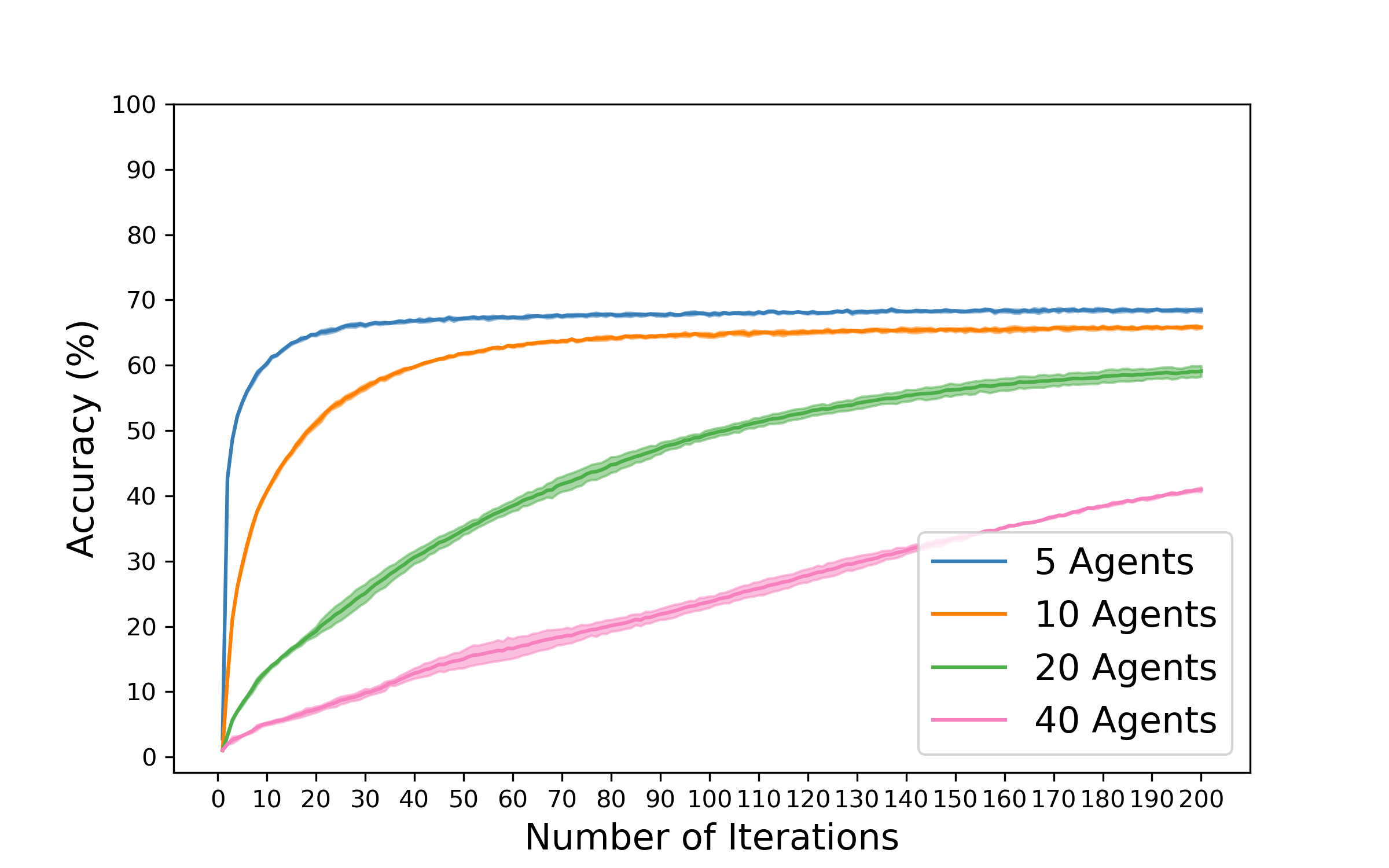}
        \caption{Scalability Analysis for FC Topology}
        \label{fig:VVIIDsc}
    \end{subfigure}
    
    \caption{Comparing algorithmic accuracy (mean$\pm$std) in fully connected (FC) (a) and ring (b) topologies with VGG16 architecture on CIFAR-100 IID data. The scalability with VGG16 architecture on CIFAR-100 IID data and fully connected topology is shown in (c).}
    \label{fig:VVIIDC}
\end{figure*}
\section{Experimental Results}
In this section, we empirically analyze the performance of DIMAT. Our code is available \href{https://github.com/nsaadati/DIMAT}{\textcolor{blue}{here on GitHub}}. We compare the effectiveness of our algorithms with other baseline decentralized algorithms such as Stochastic Gradient Push (SGP)~\cite{assran2019stochastic}, consensus-based distributed SGD (CDSGD)~\cite{jiang2017collaborative} and Cross-Gradient Aggregation (CGA)~\cite{esfandiari2021cross}. However, CGA did not converge in the explored setting and has been omitted. Further exploration is warranted to conduct a meaningful comparison involving CGA. Additionally, we introduce other baselines, Weight Matching (WM), established by Ainsworth et al.~\cite{ainsworth2022git}, and Weight Averaging (WA), inspired by consensus-based distributed Adam (CDAdam)~\cite{nazari2022dadam}. These baselines share the same setup as DIMAT but, instead of relying on activations, focus on inspecting the weights of the agents and averaging the weights of the agents, respectively. The goal in WM is to maximize the dot product between the vectorized weights of agents, formulated as a sum of bilinear assignments problem. We include these additional baselines to have a fair and direct comparison of the merging method used within the DIMAT framework. It is worth noting that WM has not been implemented for ResNet in the existing literature, and such an implementation is nontrivial. Consequently, we do not provide comparison results for WM in ResNet.  

\textbf{Experimental Setup:} Our experimental scope covered datasets such as CIFAR-100, CIFAR-10, and Tiny ImageNet. CIFAR10 and Tiny ImageNet results can be found in the Supplementary Materials. We assessed system scalability by varying the number of agents—examining scenarios with 5 to 10, 20, and 40 agents. To evaluate the algorithm's robustness across diverse network architectures, we employed three distinct model architectures: VGG16~\cite{zhang2015accelerating}, ResNet-20~\cite{he2015deep}, and ResNet50~\cite{he2015deep}. ResNet50 results can be found in the Supplementary Materials. Throughout the experiments, we initiated with pre-trained agents, each agent having undergone individual pre-training on their respective datasets for 100 epochs. All experiments were conducted five times for an average of 100 iterations to ensure result reliability. Each iteration includes 2 training epochs, where the agents are trained on their respective datasets, and a merge operation.

\textbf{Communication and Computational Overhead:} 

DIMAT notably requires less communication and computation than traditional decentralized learning methods. 
DIMAT has a merge operation once every 2 training epochs, and as such, has $0.5 * (n-1)$ communication rounds per epoch for fully connected topology and $0.5 * 2$ communication rounds per epoch for ring topology, where $n$ is the number of agents. On the other hand, methods such as CDAdam and SGP have communication steps every mini-batch, which results in significantly more communication rounds per epoch than DIMAT. 
Further visual comparisons of communication are available in the Supplementary Materials, as well as the computational results.


To evaluate DIMAT's performance, we conducted experiments comparing results under both IID and non-IID data distributions. This assessment provides insights into the algorithm's adaptability to varying data characteristics. In the IID setting, data is uniformly distributed among agents, with each agent responsible for a subset of data samples, ensuring balanced class representation. For example, with 5 agents, each handles 1/5 of the data for every class. Conversely, in the non-IID setup, significant class imbalances occur, with each agent possessing data from a smaller subset of classes. In a scenario with 100 classes and 5 agents, each agent is allocated data from 20 distinct classes, potentially creating a more challenging learning environment due to the presence of unseen data and classes.
\begin{table}[ht!]
    \caption{Comparison of test accuracy (mean$\pm$std) on CIFAR-100 with ResNet-20 architecture for 5 agents under both IID and non-IID data distribution, considering fully connected (FC) and ring topologies.}
    \label{tab:accuracy-comparison-RN20}
\resizebox{\columnwidth}{!}{%
\begin{tabular}{lllll}
\hline
\multirow{2}{*}{\textbf{Algorithm}} & \multicolumn{2}{c}{\textbf{IID}}                  & \multicolumn{2}{c}{\textbf{non-IID}}              \\ \cline{2-5} 
                                    & FC                      & Ring                    & FC                      & Ring                    \\ \hline
SGP                                 & 41.39$\pm$0.24          & 41.48$\pm$0.14          & 12.85$\pm$0.14          & 11.95$\pm$0.33          \\
CDSGD                                 & 39.69$\pm$0.20          & 39.77$\pm$0.14          & 12.31$\pm$0.19          & 10.51$\pm$0.38          \\
WA                              & 73.97$\pm$0.12          & 58.49$\pm$0.81          & \textbf{62.7$\pm$0.44}          & 8.00$\pm$0.51          \\
DIMAT (ours)                        & \textbf{73.99$\pm$0.18} & \textbf{69.07$\pm$0.01} & 
62.05$\pm$0.64& \textbf{20.45$\pm$0.65} \\ \hline
\end{tabular}%
}
\end{table}
\begin{table}[ht!]
    \caption{Comparison of test accuracy (mean$\pm$std) on CIFAR-100 with VGG16 architecture for 5 agents under both IID and non-IID data distribution, considering fully connected (FC) and ring topologies.}
    \label{tab:accuracy-comparison-VGG16}
\resizebox{\columnwidth}{!}{%
\begin{tabular}{lllll}
\hline
\multirow{2}{*}{\textbf{Algorithm}} & \multicolumn{2}{c}{\textbf{IID}}                  & \multicolumn{2}{c}{\textbf{non-IID}}              \\ \cline{2-5} 
                                    & FC                      & Ring                    & FC                      & Ring                    \\ \hline
SGP                                 & 33.42$\pm$0.25          & 33.42$\pm$0.25          & 10.00$\pm$0.17          & 10.00$\pm$0.16          \\
CDSGD                                 & 31.37$\pm$0.47          & 30.53$\pm$0.58          & 9.81$\pm$0.15          & 9.53$\pm$0.10           \\
WM                                  & 38.13$\pm$6.67          & 52.58$\pm$2.34          & 9.37$\pm$0.54           & 13.3$\pm$0.54           \\
WA                              & 66.70$\pm$0.34          & \textbf{67.02$\pm$0.24} & 49.48$\pm$0.68          & 25.15$\pm$0.27          \\
DIMAT (ours)                        & \textbf{68.08$\pm$0.24} & 63.76$\pm$0.30          & \textbf{52.99$\pm$0.24} & \textbf{25.28$\pm$0.13} \\ \hline
\end{tabular}%
}
\end{table}

\begin{figure*}[ht!]    
    \begin{subfigure}[b]{0.49\textwidth}
        \centering
        \includegraphics[width=\linewidth]{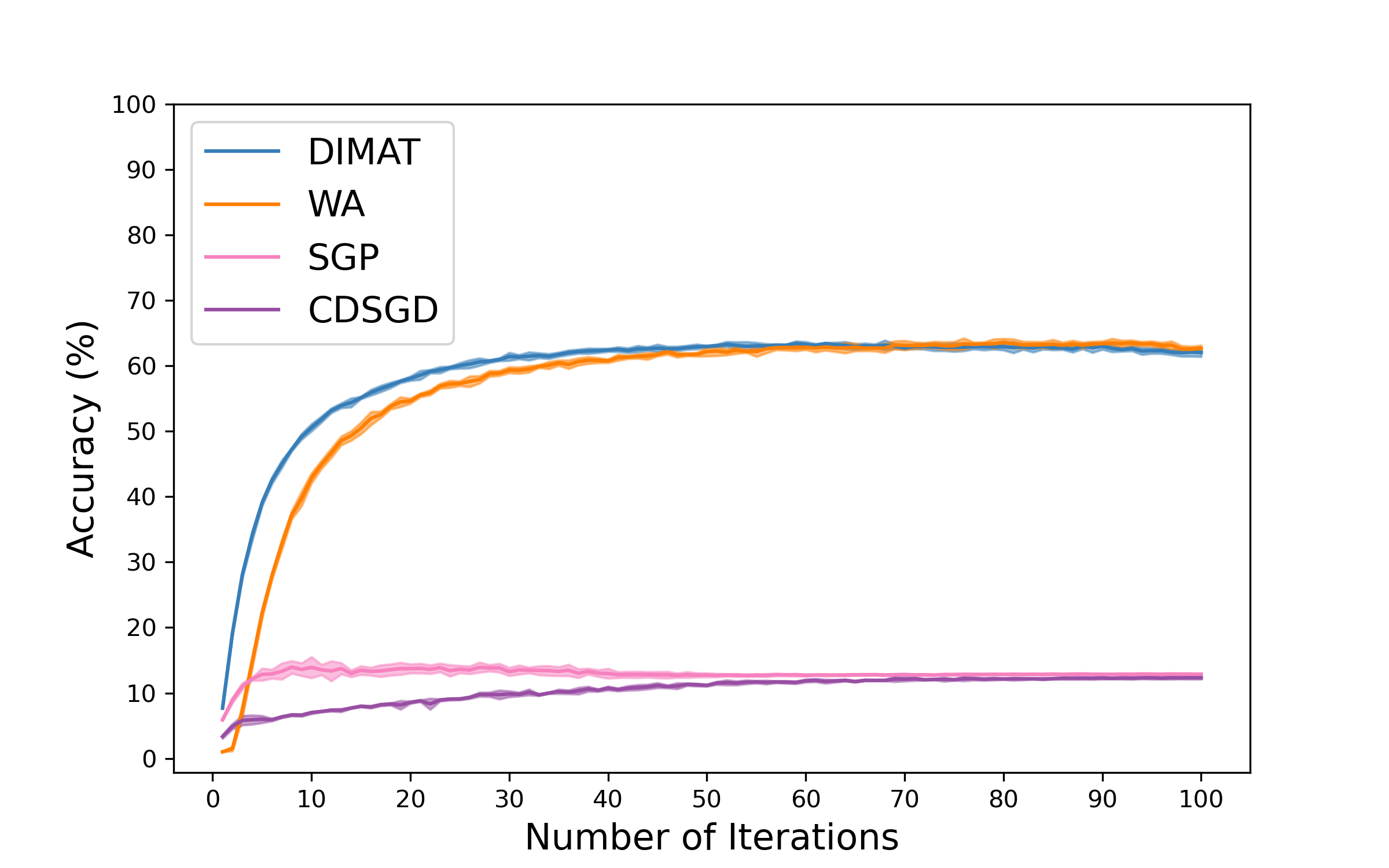}
        \caption{ResNet-20}
        \label{fig:RN25NI}
    \end{subfigure}
    \begin{subfigure}[b]{0.49\textwidth}
        \centering
        \includegraphics[width=\linewidth]{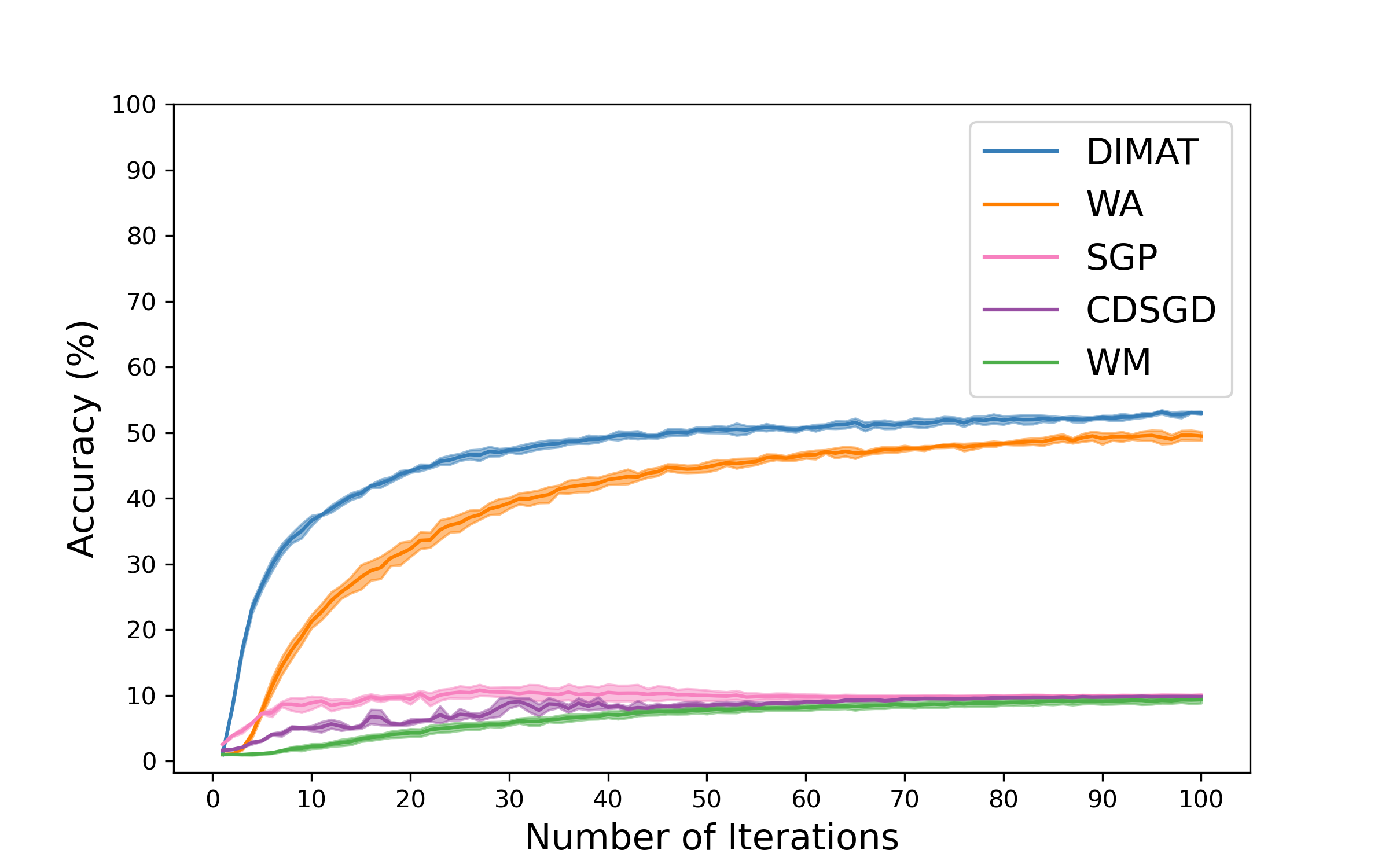}
        \caption{VGG16}
        \label{fig:V1625NI}
    \end{subfigure}
    \caption{Comparing algorithmic accuracy (mean$\pm$std) in fully connected topology with ResNet-20 (a) and VGG16 (b) architecture on CIFAR-100 non-IID data for 5 agents.}
    \label{fig:noniidacc}
\end{figure*}

\begin{figure*}[t!]
    \begin{subfigure}[b]{0.49\textwidth}
        \centering
        \includegraphics[width=\linewidth]{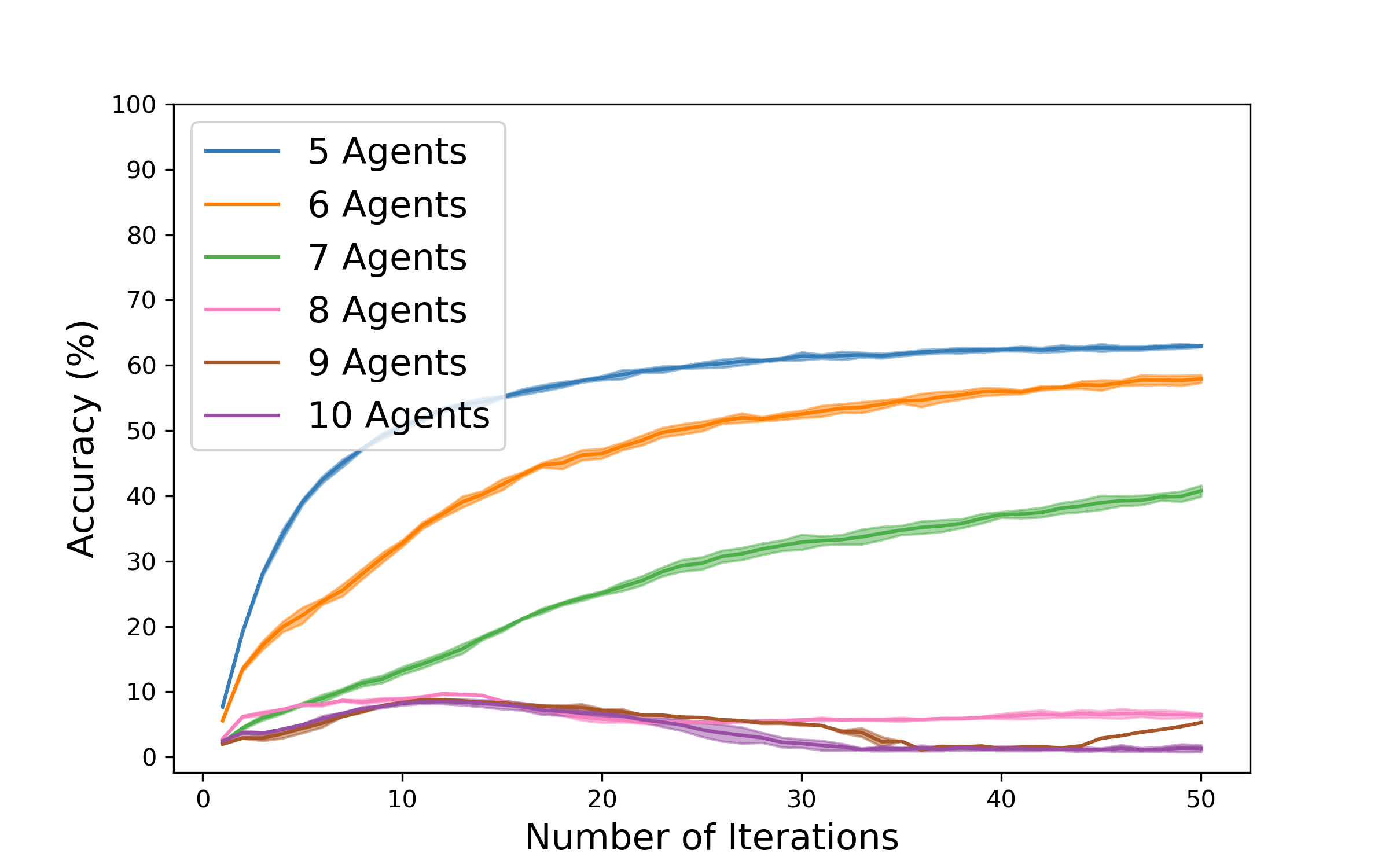}
        \caption{ResNet-20}
        \label{fig:RN2NI}
    \end{subfigure}
    \begin{subfigure}[b]{0.49\textwidth}
        \centering
        \includegraphics[width=\linewidth]{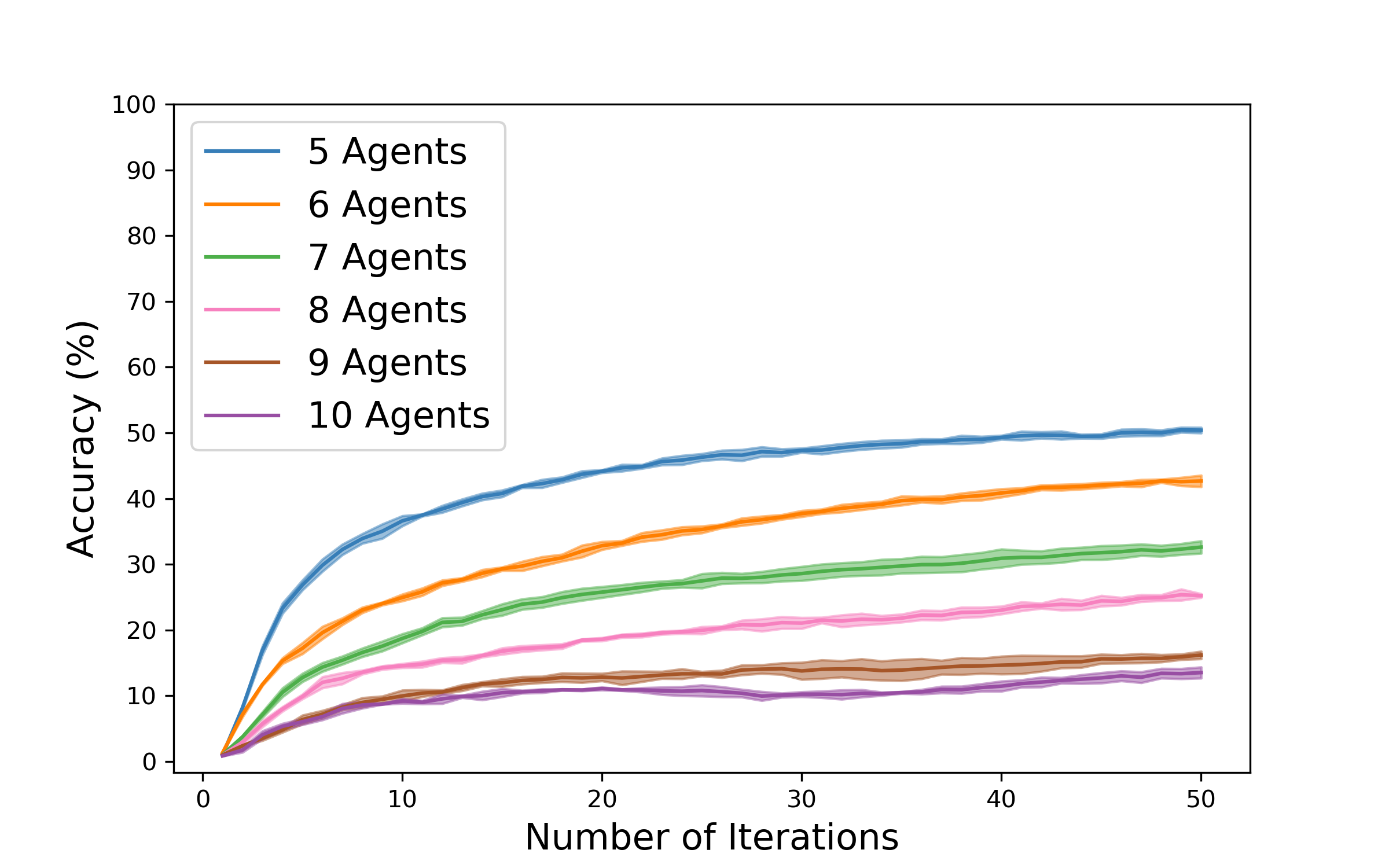}
        \caption{VGG16}
    \label{fig:V162NI}
    \end{subfigure}
    \caption{Scalability Analysis for DIMAT: Evaluating accuracy (mean$\pm$std) in fully connected topology with ResNet-20 (a) and VGG16 (b) on CIFAR-100 non-IID data across 5 to 10 agents.}
    \label{fig:noniidsca}
\end{figure*}
The results of our experiments are presented in Table \ref{tab:accuracy-comparison-RN20} for ResNet-20 and Table \ref{tab:accuracy-comparison-VGG16} for VGG16. These tables provide a comprehensive overview of testing accuracy comparisons across various network topologies under both IID and non-IID data distributions. Subsequent subsections will delve into detailed discussions regarding IID and non-IID scenarios. 

\subsection{Comparison of Algorithms in IID Setting}
In the context of the IID setting, we conducted a thorough analysis of the results presented in Tables \ref{tab:accuracy-comparison-RN20} (ResNet-20) and  \ref{tab:accuracy-comparison-VGG16} (VGG16). Our analysis reveals important insights into the performance of the different algorithms.
Visual analysis in Fig. \ref{fig:RNIIDC} and Fig. \ref{fig:VVIIDC} highlights DIMAT's rapid convergence across ResNet-20 and VGG16 in various topologies. This suggests DIMAT swiftly achieves accuracy stability during training, supporting our theoretical claims.
DIMAT consistently outperforms or matches alternative algorithms in accuracy metrics, with Weight Averaging (WA) eventually converging similarly. Scalability analysis, depicted in Fig. \ref{fig:RNIIDsc} and Fig. \ref{fig:VVIIDsc}, indicates DIMAT maintains high accuracy and convergence as agent count increases, showcasing its robustness across architectures and topologies.

\subsection{Comparison of Algorithms in Non-IID Setting}
In the non-IID setting, our analysis highlights DIMAT's consistent and swift convergence compared to alternative algorithms, particularly with fully connected topology (see Fig. \ref{fig:noniidacc}). Although DIMAT converges faster, Weight Averaging (WA) eventually achieves similar performance levels, indicating its potential effectiveness given sufficient training time. Scalability is a critical factor in evaluating the practicality of algorithms. We analyze scalability breakpoints, as shown in Fig. \ref{fig:noniidsca}, essential for informed decision-making in real-world deployments. Understanding how algorithms perform with increasing agents offers valuable insights, shedding light on the balance between individual learning and potential biases. Additionally, our Supplementary Materials delves deeper into this issue, addressing how agents' learning diminishes with more agents due to reduced dataset exposure and how leveraging larger, diverse datasets can mitigate biases.

%% file: sec/6_conclusion.tex
\section{Conclusions}
\label{sec:conclusions}
This paper presents a novel decentralized deep learning framework DIMAT by leveraging the recently developed model merging techniques. DIMAT has provably shown the sublinear convergence rate and led to the tighter error bounds, compared to popular existing approaches. The comprehensive empirical studies show the superiority of DIMAT over baselines and faster model performance gain in different settings with lower communication overhead, under both IID and non-IID data distributions, which supports well the theoretical claims. Therefore, DIMAT offers a more practical option for quickly updating (potentially large) pre-trained models with (local) new data in a collaborative manner. 
Beyond the current work, a few future research directions include: a) exploring different permutation algorithms for model merging in decentralized learning; b) combining quantization techniques with model merging to further reduce communication bits; c) resolving scalability issues of model merging, especially with non-IID data distributions. 
\section{Acknowledgements}
This work was partly supported by the National Science Foundation, USA under grants CAREER CNS-1845969, CPS Frontier CNS-1954556, CAREER CIF-2005804, SATC 2154119, and NSF/USDA-NIFA award number, 2021-67021-35329.

%% file: sec/X_suppl.tex
\clearpage
\setcounter{page}{1}
\setcounter{section}{0}
\maketitlesupplementary

\section{Additional Analysis}
\label{sec:additional_analysis}
In this section, we present additional analysis for completeness, primarily including the proof of all theorems presented in Section~\ref{sec:main_results}. Please note that the proof techniques for the proposed algorithms are different, while they share some similarities.
For the analysis, we set the merging frequency $n$ as 1 for a generic purpose.
\subsection{Algorithmic Frameworks}
\textsc{DIMAT-Adam} is slightly different from its centralized counterpart due to an auxiliary variable $\hat{\mathbf{u}}^i_k$. Based on a recent work~\cite{chen2023convergence}, the direct extension of Adam presented in~\cite{nazari2022dadam} may not necessarily converge to a stationary point.
$r_k$ in Line 4 of Algorithm~\ref{alg:dmm_adam} can take different forms, leading to different variants such as AMSGrad. In this work, we will primarily investigate the convergence rate of \textsc{DIMAT-AMSGrad}. $\oslash$ represents the division between two vectors.
\begin{algorithm}
  \caption{\textsc{DIMAT-MSGD}}
  \label{alg:dmm_msgd}
  \SetKwInOut{Input}{Input}
  \SetKwInOut{Output}{Output}
  \Input{mixing matrix $\mathbf{\Pi}$, the \# of epochs $K$, initialization $\mathbf{x}_1^i$, $\mathbf{v}_1^i$, step size $\alpha$, $0\leq\beta<1$, merging frequency $n$}
  \Output{$\bar{\mathbf{x}}_K=\frac{1}{N}\sum_{i=1}^N\mathbf{x}_K^i$}  
  \BlankLine
  \For{ $k$ in $1:K$ }
  { 
    \For{each agent $i\in\mathcal{V}$}
    {Calculate the stochastic gradient $\mathbf{g}^i_k$\;
    \eIf{$k$ mod $n$=0}
    { $\mathbf{x}_{k+1/2}^i=\sum_{j\in Nb(i)}\pi_{ij}\mathbf{P}_k^{ij}\mathbf{x}^j_{k}$\;
    }{$\mathbf{x}^i_{k+1/2}=\mathbf{x}^i_{k}$\;}
$\mathbf{v}_{k+1}^i=\beta\mathbf{v}_k^i-\alpha \mathbf{g}^i_k$\;
    $\mathbf{x}_{k+1}^i=\mathbf{x}_{k+1/2}^i+ \mathbf{v}^i_{k+1}$\;
    }
  }
\end{algorithm}

\begin{algorithm}
  \caption{\textsc{DIMAT-Adam}}
  \label{alg:dmm_adam}
  \SetKwInOut{Input}{Input}
  \SetKwInOut{Output}{Output}
  \Input{mixing matrix $\mathbf{\Pi}$, the \# of epochs $K$, initialization $\mathbf{x}_1$, $\mathbf{m}_0^i=\mathbf{v}_0^i=0$, $\hat{\mathbf{u}}^i_1=\mathbf{v}^i_1$, step size $\alpha$, merging frequency $n$, small positive constant $\epsilon$, $\beta_1\in[0,1)$}
  \Output{$\bar{\mathbf{x}}_K=\frac{1}{N}\sum_{i=1}^N\mathbf{x}_K^i$}  
  \BlankLine
  \For{ $k$ in $1:K$ }
  { 
    \For{each agent $i\in\mathcal{V}$}
    {Calculate the stochastic gradient $\mathbf{g}^i_k$\;
$\mathbf{m}_{k}^i=\beta_1\mathbf{m}_{k-1}^i+(1-\beta_1) \mathbf{g}^i_k$\;
    $\mathbf{v}^i_k=r_k(\mathbf{g}^i_1,...,\mathbf{g}^i_k)$\;
    \eIf{$k$ mod $n$=0}
    { $\mathbf{x}_{k+1/2}^i=\sum_{j\in Nb(i)}\pi_{ij}\mathbf{P}_k^{ij}\mathbf{x}^j_{k}$\; $\hat{\mathbf{u}}_{k+1/2}^i=\sum_{j\in Nb(i)}\pi_{ij}\mathbf{P}_k^{ij}\hat{\mathbf{u}}^j_{k}$\;
    }{$\mathbf{x}^i_{k+1/2}=\mathbf{x}^i_{k}$\; $\hat{\mathbf{u}}^i_{k+1/2}=\hat{\mathbf{u}}^i_{k}$\;}
    $\mathbf{u}^i_k=\text{max}(\hat{\mathbf{u}}^i_k,\epsilon)$\;
    $\mathbf{x}^i_{k+1}=\mathbf{x}^i_{k+1/2}-\alpha\mathbf{m}_{k}^i\oslash (\mathbf{u}^i_k)^{1/2}$\;
    $\hat{\mathbf{u}}^i_{k+1}=\hat{\mathbf{u}}^i_{k+1/2}-\mathbf{v}^i_{k-1}+\mathbf{v}^i_k$\;
    }
  }
\end{algorithm}
\subsection{Additional Theoretical Results}
\begin{theorem}\label{dmm-msgd-theo}
    Let Assumptions~\ref{assumption_1} and~\ref{assumption_3} hold. If the step size $\alpha\leq \textnormal{min}\{\frac{(1-\sqrt{\rho'})(1-\beta)}{4L},\frac{(1-\sqrt{\rho'})^2(1-\beta)^2}{6L}\}$ in Algorithm~\ref{alg:dmm_msgd}, then for all $K\geq 1$, the following relationship holds true:
    \begin{equation}
    \begin{split}
        &\frac{1}{K}\sum_{k=1}^K\mathbb{E}[\|\nabla f(\bar{\mathbf{x}}_k)\|^2]\leq \frac{2(1-\beta)(f(\bar{\mathbf{x}}_0)-f^*)}{\alpha K}+\frac{L\alpha\sigma^2}{(1-\beta)^2N}\\&+\frac{4\alpha^2\sigma^2L^2}{(1-\beta)^2(1-\rho')}+\frac{16\alpha^2\kappa^2L^2}{(1-\beta)^2(1-\sqrt{\rho'})^2},
    \end{split}
    \end{equation}
where $\bar{\mathbf{x}}_k=\frac{1}{N}\sum_{i=1}^N\mathbf{x}^i_k$.
\end{theorem}
\begin{corollary}\label{dmm-msgd-coro}
        Let Assumptions~\ref{assumption_1} and~\ref{assumption_3} hold. If step size $\alpha=\mathcal{O}(\sqrt{\frac{N}{K}})$ in Algorithm~\ref{alg:dmm_msgd}, then for all $K\geq\textnormal{max}\{\frac{32NL^2}{(1-\beta)^2(1-\sqrt{\rho'})^2},\frac{36NL^2}{(1-\beta)^4(1-\sqrt{\rho'})^4}\}$, we have
        \begin{equation}
        \begin{split}
            \frac{1}{K}\sum_{k=1}^K\mathbb{E}[\|\nabla f(\bar{\mathbf{x}}_k)\|^2]&\leq \mathcal{O}(\sqrt{\frac{1}{NK}}+\frac{N}{(1-\rho')K}\\&+\frac{N}{(1-\sqrt{\rho'})^2K}).
        \end{split}
        \end{equation}
\end{corollary}
Before presenting the main result for \textsc{DMM-Adam}, we define specifically $r_k$ in Algorithm~\ref{alg:dmm_adam} as $\hat{\mathbf{v}}^i_k=\beta_2\hat{\mathbf{v}}^i_{k-1}+(1-\beta_2)\mathbf{g}^i_k\odot\mathbf{g}^i_k$ and $\mathbf{v}^i_k = \textnormal{max}\{\hat{\mathbf{v}}^i_k, \mathbf{v}^i_{k-1}\}$, where $0\leq\beta_2<1$ and $\hat{\mathbf{v}}^i_0=0$, leading to \textsc{DIMAT-AMSGrad}. To show the convergence rate, we need another assumption specifically for the adaptive gradient descent type of algorithms, which bounds the infinity norms of $\mathbf{g}^i_k$ and $\nabla f_i(\mathbf{x}^i_k)$ by a positive constant $G_{\infty}<\infty$. This assumption has actually been relaxed in many first-order methods such as SGD and MSGD types~\cite{yu2019linear}. However, the relaxation of the assumption in adaptive gradient methods is out of our scope and we will still proceed with this assumption.
\begin{theorem}\label{dmm-adam-theo}
   Let Assumptions~\ref{assumption_1} and~\ref{assumption_3} hold. Also suppose that $\|\mathbf{g}_k^i\|_\infty\leq G_\infty$ and that  $\|\nabla f^i(\mathbf{x}^i_k)\|_\infty\leq G_\infty$ for all $i\in\mathcal{V}$ and $k\geq 1$. If step size $\alpha=\mathcal{O}(\frac{1}{\sqrt{Kd}})$ in Algorithm~\ref{alg:dmm_adam}, then for all $K\geq \frac{256L^2}{d\epsilon}$, we have,
   \begin{equation}
   \begin{split}
       \frac{1}{K}\sum_{k=1}^K\mathbb{E}[\|\nabla f(\bar{\mathbf{x}}_k)\|^2]&\leq\mathcal{O}(\frac{d^{1.5}}{\sqrt{NK}}+\frac{dN}{(1-\sqrt{\rho'})^2K}\\&+\frac{N^{1.5}d^{0.5}}{K^{1.5}}+\frac{\sqrt{N}d^2}{(1-\sqrt{\rho'})K}\\&+\frac{Nd^{1.5}}{(1-\sqrt{\rho'})K^{1.5}})
   \end{split}
   \end{equation}
where $\bar{\mathbf{x}}_k=\frac{1}{N}\sum_{i=1}^N\mathbf{x}^i_k$.
\end{theorem}

Theorem~\ref{dmm-adam-theo} shows that Algorithm~\ref{alg:dmm_adam} converges with a rate of $\mathcal{O}(\frac{d^{1.5}}{\sqrt{NK}})$ when $K$ is sufficiently large. Also, \textsc{DIMAT-AMSGrad} enjoys the linear speed up as \textsc{DIMAT-SGD} and \textsc{DIMAT-MSGD}. The dependence on the dimension of $\mathbf{x}$ is attributed to the bounded assumption of the infinity norms of gradients. The interaction between topology and model merging from Remark~\ref{remark_2} can comparably apply to \textsc{DIMAT-AMSGrad} in this context. The transition between the transient and stable regimes depending on when $\mathcal{O}(\frac{d^{1.5}}{\sqrt{NK}})$ dominates also motivates the further future investigation of convergence dynamics. 

\subsection{Additional Analysis for \textsc{DIMAT-SGD}}
With abuse of notation, we use some upper bold characters to represent vectors after they are expanded. 
Define
\[\mathbf{X}_k=[\mathbf{x}^1_k;\mathbf{x}^2_k;...;\mathbf{x}^N_k]^\top\in\mathbb{R}^{dN},\]\[\mathbf{G}_k=[\mathbf{g}^1_k;\mathbf{g}^2_k;...;\mathbf{g}^N_k]^\top\in\mathbb{R}^{dN},\]\[\mathbf{H}_k=[\nabla f^1(\mathbf{x}^1_k);\nabla f^2(\mathbf{x}^2_k);...;\nabla f^N(\mathbf{x}^N_k))]^\top\in\mathbb{R}^{dN},\]\[\mathbf{Q}=\frac{1}{dN}\mathbf{1}\mathbf{1}^\top_{dN}\in\mathbb{R}^{dN\times dN}\]
Without loss of generality, suppose that the initialization $\mathbf{X}_0=\mathbf{0}$ throughout the rest of analysis. For \textsc{DIMAT-SGD}, we have
\begin{equation}
    \mathbf{X}_k = -\alpha\sum_{\tau=1}^{k-1}\prod_{t=\tau+1}^{k-1}\mathbf{W}\mathbf{P}_t\mathbf{G}_\tau
\end{equation}
For ease of exposition, we define $\prod_{\tau=k+1}^k\mathbf{W}\mathbf{P}_\tau=\mathbf{I}$ in our analysis. Left multiplying by $\mathbf{I-Q}$ yields the following relationship
\begin{equation}\label{eq_14}
    (\mathbf{I-Q})\mathbf{X}_k = -\alpha\sum_{\tau=1}^{k-1}(\mathbf{I-Q})\prod_{t=\tau+1}^{k-1}\mathbf{W}\mathbf{P}_t\mathbf{G}_\tau,
\end{equation}
which will serve to characterize the optimal error bound. By taking the squared norm and expectation on both sides, we have
\begin{equation}\label{consensus}
    \mathbb{E}[\|(\mathbf{I-Q})\mathbf{X}_k\|^2]=\alpha^2\mathbb{E}[\|\sum_{\tau=1}^{k-1}(\mathbf{I-Q})\prod_{t=\tau+1}^{k-1}\mathbf{W}\mathbf{P}_t\mathbf{G}_\tau\|^2].
\end{equation}
The left side of above equation is equivalent to $\mathbb{E}[\frac{1}{N}\sum_{i=1}^N\|\mathbf{x}^i_k-\mathbf{\bar{x}}_k\|^2]$. To further analyze the Eq.~\ref{consensus}, we investigate its right side in the following.
\begin{equation}\label{eq_16}
\begin{split}
    &\alpha^2\mathbb{E}[\|\sum_{\tau=1}^{k-1}(\mathbf{I-Q})\prod_{t=\tau+1}^{k-1}\mathbf{W}\mathbf{P}_t\mathbf{G}_\tau\|^2]\leq \\&2\alpha^2\underbrace{\mathbb{E}[\|\sum_{\tau=1}^{k-1}(\mathbf{I-Q})\prod_{t=\tau+1}^{k-1}\mathbf{W}\mathbf{P}_t(\mathbf{G}_\tau-\mathbf{H}_\tau)\|^2]}_{T_1}\\&
    +2\alpha^2\underbrace{\mathbb{E}[\|\sum_{\tau=1}^{k-1}(\mathbf{I-Q})\prod_{t=\tau+1}^{k-1}\mathbf{W}\mathbf{P}_t\mathbf{H}_\tau\|^2]}_{T_2},
\end{split}
\end{equation}
which follows by using the basic inequality $\|\mathbf{a}+\mathbf{b}\|^2\leq 2\|\mathbf{a}\|^2+2\|\mathbf{b}\|^2$. We will next study the upper bounds for $T_1$ and $T_2$, respectively. Before that, we present some technical detail for how to derive $\rho'\leq \rho$ and then state two key lemmas to manipulate $(\mathbf{I-Q})\prod_{\tau=1}^{k}\mathbf{W}\mathbf{P}_\tau$ and $\mathbf{G}_k - \mathbf{H}_k$.

\textbf{Analysis of $\rho'\leq \rho$.} As $\mathbf{W}\mathbf{P}_k$ is symmetric, the immediate outcome is that the singular values of $\mathbf{W}\mathbf{P}_k$ are equal to the absolute values of eigenvalues of $\mathbf{W}\mathbf{P}_k$, which results in $\zeta_l(\mathbf{W}\mathbf{P}_k)=|\lambda_l(\mathbf{W}\mathbf{P}_k)|$, where $\zeta_l$ is the $l$-th singular value of $\mathbf{W}\mathbf{P}_k$. This result is well-known and we skip the proof in this context. According to the Courant–Fischer–Weyl Min-Max Principle~\cite{il2022min}, the following relationship can be obtained:
\begin{equation}
\begin{split}
&\zeta_l(\mathbf{W}\mathbf{P}_k)=\textnormal{max}_{S:dim(S)=l}\textnormal{min}_{x\in S,\|x\|=1}\|\mathbf{W}\mathbf{P}_kx\|\\&\leq \textnormal{max}_{S:dim(S)=l}\textnormal{min}_{x\in S,\|x\|=1}\|\mathbf{W}\|\|\mathbf{P}_kx\|\\&=\zeta_1(\mathbf{W})\cdot\textnormal{max}_{S:dim(S)=l}\textnormal{min}_{x\in S,\|x\|=1}\|\mathbf{P}_kx\|\\&\leq\zeta_1(\mathbf{W})\zeta_i(\mathbf{P}_k),
\end{split}
\end{equation}
where $S:dim(S)=l$ is a subspace of $\mathbb{R}^{dN}$ of dimension $l$.
Then,
\begin{equation}
   \begin{split}&\zeta_l(\mathbf{W}\mathbf{P}_k)= \zeta_l([\mathbf{W}\mathbf{P}_k]^\top)\\&=\zeta_l(\mathbf{P}^\top_k\mathbf{W}^\top)\leq \zeta_1(\mathbf{P}^\top_k)\zeta_l(\mathbf{W}^\top)=\zeta_l(\mathbf{W})\zeta_1(\mathbf{P}_k).
   \end{split}
\end{equation}
We have known that $\mathbf{W}$, $\mathbf{P}_k$ and $\mathbf{W}\mathbf{P}_k$ are symmetric such that
\begin{equation}
    |\lambda_l(\mathbf{W}\mathbf{P}_k)|\leq|\lambda_l(\mathbf{W})||\lambda_1(\mathbf{P}_k)|
\end{equation}
Since all eigenvalues of $\mathbf{P}_k$ are contained in the roots of unity, the modulus of any eigenvalue of $\mathbf{P}_k$ is 1, i.e., $|\lambda_1(\mathbf{P}_k)|=1$. With this in hand, we have
\begin{equation}
    |\lambda_l(\mathbf{W}\mathbf{P}_k)|\leq|\lambda_l(\mathbf{W})|
\end{equation}
The above inequality implies that 
\begin{equation}
\begin{split}
    &\textnormal{max}\{|\lambda_2(\mathbf{W}\mathbf{P}_k)|,|\lambda_{dN}(\mathbf{W}\mathbf{P}_k)|\}\leq\\& \textnormal{max}\{|\lambda_2(\mathbf{W})|,|\lambda_{dN}(\mathbf{W})|\},
\end{split}
\end{equation}
which ensures the fact that $\sqrt{\rho'}\leq \sqrt{\rho}$.

\begin{lemma}\label{lemma_1}
    Let Assumption~\ref{assumption_1} hold. Suppose that $\mathbf{E}[\mathbf{g}^i] = \nabla f^i(\mathbf{x}^i), \forall i\in\mathcal{V}$. Then, we have the following relationship
    \begin{equation}
        \mathbb{E}[\|\frac{1}{N}\sum_{i=1}^N\mathbf{g}^i\|^2]\leq\frac{1}{N}\sigma^2+\mathbb{E}[\|\frac{1}{N}\sum_{i=1}^N\nabla f^i(\mathbf{x}^i)\|^2].
    \end{equation}
\end{lemma}
The proof for the Lemma~\ref{lemma_1} follows similarly Lemma 1 in~\cite{yu2019linear} and we skip it in this context.
\begin{lemma}\label{lemma_2}
    Let Assumption~\ref{assumption_3} hold. Then, for any integer $k\geq 1$, we have $\|(\mathbf{I-Q})\prod_{\tau=1}^{k}\mathbf{W}\mathbf{P}_\tau\|\leq(\sqrt{\rho'})^k<1$, where $\|\cdot\|$ denotes the spectral norm in this context.
\end{lemma}
\begin{proof}
The proof can be easily obtained by using Assumption~\ref{assumption_3} and the similar analysis techniques in Lemma IV.2 in~\cite{berahas2018balancing}.
\end{proof}
We are now able to bound $T_1$ and $T_2$. By following the similar proof techniques and adapting the analysis in~\cite{yu2019linear}, the following bounds are obtained accordingly.
\begin{equation}\label{eq_21}
    T_1\leq \frac{N\sigma^2}{1-\rho'}
\end{equation}
\begin{equation}\label{eq_22}
\begin{split}
    &T_2\leq \frac{1}{1-\sqrt{\rho'}}[(8L^2\sum^k_{\tau=1}(\rho')^{(k-\tau)/2}\mathbb{E}[\|\mathbf{(I-Q)\mathbf{X}_\tau}\|^2]\\&
    +4N\sum_{\tau=1}^k(\rho')^{(k-\tau)/2}\mathbb{E}[\|\frac{1}{N}\sum_{i=1}^N\nabla f^i(\mathbf{x}^i_\tau)\|^2]]\\&+\frac{4N\kappa^2}{1-\sqrt{\rho'}}
\end{split}
\end{equation}
Based on the upper bounds for $T_1$ and $T_2$, we can obtain the upper bound for $\frac{1}{N}\sum_{i=1}^N\mathbb{E}[\|\mathbf{x}^i_k-\mathbf{\bar{x}}_k\|^2]$.
\begin{lemma}\label{lemma_3}
    Let Assumptions~\ref{assumption_1} and~\ref{assumption_3} hold. For $\mathbf{x}_k^i$ defined by Algorithm~\ref{alg:dmm_sgd},
    if step size $\alpha\leq \frac{1-\sqrt{\rho'}}{4\sqrt{2}L}$, then $\forall K\geq 1$, the following relationship holds true
    \begin{equation}
    \begin{split}
        &\sum_{k=1}^K\frac{1}{N}\sum_{i=1}^N\mathbb{E}[\|\mathbf{x}^i_k-\mathbf{\bar{x}}_k\|^2]\leq\frac{4K\alpha^2\sigma^2}{1-\rho'}\\&+\frac{16\alpha^2}{(1-\sqrt{\rho'})^2}\sum^K_{k=1}\mathbb{E}[\|\frac{1}{N}\sum_{i=1}^N\nabla f^i(\mathbf{x}^i_k)\|^2]\\&+\frac{16K\alpha^2\kappa^2}{(1-\sqrt{\rho'})^2},
    \end{split}
    \end{equation}
where $\bar{\mathbf{x}}_k=\frac{1}{N}\sum_{i=1}^N\mathbf{x}^i_k$.
\end{lemma}
Combining Eqs~\ref{eq_16},~\ref{eq_21}, and~\ref{eq_22} and summing $k$ over $\{1,2,...,K\}$, and following the proof techniques from~\cite{yu2019linear} can complete the proof for Lemma~\ref{lemma_3}.
With Lemma~\ref{lemma_3} in hand, we are ready to give the proof of Theorem~\ref{dmm-sgd-theo} in the following.
\begin{proof}
    We start with the descent inequality given by the smoothness of $f$ such that
    \begin{equation}\label{eq_24}
    \begin{split}
        &\mathbb{E}[f(\bar{\mathbf{x}}_{k+1})]\leq \mathbb{E}[f(\bar{\mathbf{x}}_k)] + \mathbb{E}[\langle\nabla f(\bar{\mathbf{x}}_k),\bar{\mathbf{x}}_{k+1}-\bar{\mathbf{x}}_k\rangle]\\&+\frac{L}{2}\mathbb{E}[\|\bar{\mathbf{x}}_{k+1}-\bar{\mathbf{x}}_k\|^2]
    \end{split}
    \end{equation}
We first process the second term on the right side of above inequality. Replacing $\bar{\mathbf{x}}_{k+1}-\bar{\mathbf{x}}_k$ with $-\alpha\frac{1}{N}\sum_{i=1}^N\mathbf{g}^i(\mathbf{x}^i_k)$ allows us to study $-\alpha\mathbb{E}[\langle\nabla f(\bar{\mathbf{x}}_k),\frac{1}{N}\sum_{i=1}^N\mathbf{g}^i(\mathbf{x}^i_k)\rangle]$.
Thus, we have
\begin{equation}
\begin{split}
&\langle\nabla f(\bar{\mathbf{x}}_k),\frac{1}{N}\sum_{i=1}^N\mathbf{g}^i(\mathbf{x}^i_k)\rangle=\frac{1}{2}(\|\nabla f(\bar{\mathbf{x}}_k)\|^2+\\&\|\frac{1}{N}\sum_{i=1}^N\nabla f^i(\mathbf{x}^i_k)\|^2-\|\nabla f(\bar{\mathbf{x}}_k)-\frac{1}{N}\sum_{i=1}^N\nabla f^i(\mathbf{x}_k^i)\|^2)\\&\geq \frac{1}{2}(\|\nabla f(\bar{\mathbf{x}}_k)\|^2+\|\frac{1}{N}\sum_{i=1}^N\nabla f^i(\mathbf{x}^i_k)\|^2\\&-L^2\frac{1}{N}\sum_{i=1}^N\|\bar{\mathbf{x}}_k-\mathbf{x}^i_k\|^2).
\end{split}
\end{equation}
The last inequality follows from the smoothness assumption. Therefore, we have
\begin{equation}
\begin{split}
&\mathbb{E}[\langle\nabla f(\bar{\mathbf{x}}_k),\bar{\mathbf{x}}_{k+1}-\bar{\mathbf{x}}_k\rangle]\leq-\frac{\alpha}{2}\mathbb{E}[\|\nabla f(\bar{\mathbf{x}}_k)\|^2\\&+\|\frac{1}{N}\sum_{i=1}^N\nabla f^i(\mathbf{x}^i_k)\|^2]+\frac{\alpha L^2}{2}\frac{1}{N}\sum_{i=1}^N\mathbb{E}[\|\bar{\mathbf{x}}_k-\mathbf{x}^i_k\|^2].
\end{split}
\end{equation}
With Eq.~\ref{eq_24}, the following relationship can be obtained.
\begin{equation}
\begin{split}
&\mathbb{E}[f(\bar{\mathbf{x}}_{k+1})]\leq \mathbb{E}[f(\bar{\mathbf{x}}_k)]-\frac{\alpha}{2}\mathbb{E}[\|\nabla f(\bar{\mathbf{x}}_k)\|^2\\&+\|\frac{1}{N}\sum_{i=1}^N\nabla f^i(\mathbf{x}^i_k)\|^2]+\frac{\alpha L^2}{2}\frac{1}{N}\sum_{i=1}^N\mathbb{E}[\|\bar{\mathbf{x}}_k-\mathbf{x}^i_k\|^2]\\&+\frac{L}{2}\mathbb{E}[\|\frac{1}{N}\sum_{i=1}^N\mathbf{g}^i_k\|^2].
\end{split}
\end{equation}
The last inequality holds due to 
$\bar{\mathbf{x}}_{k+1}-\bar{\mathbf{x}}_k=\frac{1}{N}\sum_{i=1}^N\mathbf{g}^i_k$.
With Lemma~\ref{lemma_1} and some mathematical manipulations, we have 
\begin{equation}
    \begin{split}
&\mathbb{E}[f(\bar{\mathbf{x}}_{k+1})]\leq\mathbb{E}[f(\bar{\mathbf{x}}_k)]-\frac{\alpha}{2}\mathbb{E}[\|\nabla f(\bar{\mathbf{x}}_k)\|^2\\&+\|\frac{1}{N}\sum_{i=1}^N\nabla f^i(\mathbf{x}^i_k)\|^2]+\frac{\alpha L^2}{2}\frac{1}{N}\sum_{i=1}^N\mathbb{E}[\|\bar{\mathbf{x}}_k-\mathbf{x}^i_k\|^2]\\&+\frac{L\alpha^2}{2}(\frac{\sigma^2}{N}+\mathbb{E}[\frac{1}{N}\sum_{i=1}^N\|\nabla f^i(\mathbf{x}^i_k)\|^2])\\&=\mathbb{E}[f(\bar{\mathbf{x}}_k)]-\frac{\alpha}{2}\mathbb{E}[\|\nabla f(\bar{\mathbf{x}}_k)\|^2]-\\&(\frac{\alpha}{2}-\frac{L\alpha^2}{2})\mathbb{E}[\|\frac{1}{N}\sum_{i=1}^N\nabla f^i(\mathbf{x}^i_k)\|^2]\\& + \frac{\alpha L^2}{2N}\sum_{i=1}^N\mathbb{E}[\|\bar{\mathbf{x}}_k-\mathbf{x}^i_k\|^2]+\frac{L\alpha^2\sigma^2}{2N},
    \end{split}
\end{equation}
which implies the following inequality
\begin{equation}\label{eq_29}
\begin{split}
&\mathbb{E}[\|\nabla f(\bar{\mathbf{x}}_k)\|^2]\leq\frac{2}{\alpha}(\mathbb{E}[f(\bar{\mathbf{x}}_k)]-\mathbb{E}[f(\bar{\mathbf{x}}_{k+1})])\\&-(1-L\alpha)\mathbb{E}[\|\frac{1}{N}\sum_{i=1}^N\nabla f^i(\mathbf{x}^i_k)\|^2]+\frac{L^2}{N}\mathbb{E}[\|\bar{\mathbf{x}}_k-\mathbf{x}^i_k\|^2]\\&+\frac{L\alpha\sigma^2}{N}.
\end{split}
\end{equation}
The above relationship is obtained by dividing $\alpha/2$ on both sides.
Summing $k$ over $\{1,2,...,K\}$ yields
\begin{equation}
\begin{split}
    &\sum_{k=1}^K\mathbb{E}[\|\nabla f(\bar{\mathbf{x}}_k)\|^2]\leq \frac{2}{\alpha}(f(\bar{\mathbf{x}}_0)-\mathbb{E}[f(\bar{\mathbf{x}}_K)])\\&-(1-L\alpha)\sum_{k=1}^K\mathbb{E}[\|\frac{1}{N}\sum_{i=1}^N\nabla f^i(\mathbf{x}^i_k)\|^2]\\&+L^2(\frac{4K\alpha^2\sigma^2}{1-\rho'}+\frac{16\alpha^2}{(1-\sqrt{\rho'})^2}\sum_{k=1}^K\mathbb{E}[\|\frac{1}{N}\sum_{i=1}^N\nabla f^i(\mathbf{x}^i_k)\|^2]\\&+\frac{16K\alpha^2\kappa^2}{(1-\sqrt{\rho'})^2})+\frac{L\alpha\sigma^2K}{N}\\&\leq \frac{2}{\alpha}(f(\bar{\mathbf{x}}_0)-f^*)-\\&(1-L\alpha-\frac{16L^2\alpha^2}{(1-\sqrt{\rho'})^2})\sum_{k=1}^K\mathbb{E}[\|\frac{1}{N}\sum_{i=1}^N\nabla f^i(\mathbf{x}^i_k)\|^2]\\&
    +\frac{4KL^2\alpha^2\sigma^2}{1-\rho'}+\frac{16KL^2\alpha^2\kappa^2}{(1-\sqrt{\rho'})^2}+\frac{LK\alpha\sigma^2}{N}
\end{split}
\end{equation}
The last inequality is attained by substituting the conclusion from Lemma~\ref{lemma_3} into Eq.~\ref{eq_29}. 
Due to the condition for the step size $\alpha$, we know that $1-L\alpha-\frac{16L^2\alpha^2}{(1-\sqrt{\rho'})^2}\geq0$, which would simplify the right side in the last inequality. Hence,
\begin{equation}
\begin{split}
&\sum_{k=1}^K\mathbb{E}[\|\nabla f(\bar{\mathbf{x}}_k)\|^2]\leq \frac{2}{\alpha}(f(\bar{\mathbf{x}}_0)-f^*)+\frac{4KL^2\alpha^2\sigma^2}{1-\rho'}\\&+\frac{16KL^2\alpha^2\kappa^2}{(1-\sqrt{\rho'})^2}+\frac{LK\alpha\sigma^2}{N}
\end{split}
\end{equation}
The desirable result is obtained by dividing $K$ on both sides.
\end{proof}
The proof for Corollary~\ref{dmm-sgd-coro} is easily completed by substituting the step size $\alpha=\mathcal{O}(\sqrt{\frac{N}{K}})$ into the error bound in Theorem~\ref{dmm-sgd-theo}.
\subsection{Additional Analysis for \textsc{DIMAT-MSGD}}
To prove Theorem~\ref{dmm-msgd-theo}, we need another auxiliary variable to assist in establishing the relationship between two consecutive steps of $\bar{\mathbf{x}}$. By multiplying $\frac{1}{N}\mathbf{1}\mathbf{1}^\top$ on \[\mathbf{v}_{k+1}^i=\beta\mathbf{v}_k^i-\alpha\mathbf{g}^i_k,\mathbf{x}_{k+1}^i=\mathbf{x}^i_{k+1/2}+\mathbf{v}^i_{k+1},\]we obtain
\begin{equation}
\begin{split}
    &\bar{\mathbf{v}}_{k+1}=\beta\bar{\mathbf{v}}_{k}-\alpha\frac{1}{N}\sum_{i=1}^N\mathbf{g}^i_k\\&
    \bar{\mathbf{x}}_{k+1}=\bar{\mathbf{x}}_{k}+\bar{\mathbf{v}}_{k+1},
\end{split}
\end{equation}
which follows by the approximate equivalence between averaged permuted parameters and averaged parameters in Remark~\ref{remark_1}. To characterize the analysis, we define an auxiliary variable in the following
\begin{equation}\label{auxi_rel}
    \bar{\mathbf{p}}_k:=\frac{1}{1-\beta}\bar{\mathbf{x}}_k-\frac{\beta}{1-\beta}\bar{\mathbf{x}}_{k-1},
\end{equation}
where $k\geq 1$. If $k=0$, then $\bar{\mathbf{p}}_0=\bar{\mathbf{x}}_0=0$. The following lemma states the relationship between $\bar{\mathbf{p}}_{k+1}$ and $\bar{\mathbf{p}}_k$. 
\begin{lemma}\label{lemma_4}
    Let $\bar{\mathbf{p}}_k$ be defined in Eq.~\ref{auxi_rel}. Based on Algorithm~\ref{alg:dmm_msgd}, we have
    \begin{equation}
        \bar{\mathbf{p}}_{k+1}-\bar{\mathbf{p}}_{k}=-\frac{\alpha}{(1-\beta)N}\sum_{i=1}^N\mathbf{g}^i_k.
    \end{equation}
\end{lemma}
The proof of Lemma~\ref{lemma_4} follows similarly from Lemma 3 in~\cite{esfandiari2021cross}. We next study the relationship between $\bar{\mathbf{p}}_k$ and $\bar{\mathbf{x}}_k$. 
\begin{lemma}\label{lemma_5}
Based on Algorithm~\ref{alg:dmm_msgd} and $\bar{\mathbf{p}}_k$ in Eq.~\ref{auxi_rel}, the following relationship holds true for all $K\geq 1$
\begin{equation}
\sum_{k=1}^K\|\bar{\mathbf{p}}_k-\bar{\mathbf{z}}_k\|\leq \frac{\alpha^2\beta^2}{(1-\beta)^4}\sum_{k=1}^K\|\frac{1}{N}\sum_{i=1}^N\mathbf{g}^i_k\|^2.
\end{equation}
\end{lemma}
The proof of Lemma~\ref{lemma_5} can be adapted from that of Lemma 4 in~\cite{esfandiari2021cross}. Next, we define \[\mathbf{V}_k=[\mathbf{v}^1_k;\mathbf{v}^2_k;...;\mathbf{v}^N_k]^\top\in\mathbb{R}^{dN}.\]To prove Theorem~\ref{dmm-msgd-theo}, we need to get the upper bound for $\mathbb{E}[\|(\mathbf{I-Q})\mathbf{X}_k\|^2]$ first, as done for \textsc{DIMAT-SGD}. Hence, we proceed with expanding $\|(\mathbf{I-Q})\mathbf{X}_k\|^2$. 
Recursively applying $\mathbf{v}_{k+1}^i=\beta\mathbf{v}_k^i-\alpha\mathbf{g}^i_k$ and set $\mathbf{V}_0 = 0$ yields
\begin{equation}\label{eq_36}
    \mathbf{V}_k=-\alpha\sum_{\tau=0}^{k-1}\beta^{k-1-\tau}\mathbf{G}_\tau
\end{equation}
With $\mathbf{X}_k=\mathbf{W}\mathbf{P}_{k-1}\mathbf{X}_{k-1}+\mathbf{V}_k$, recursively applying it and using 0 initial condition attains 
\begin{equation}\label{eq_37}
\mathbf{X}_k=\sum_{\tau=1}^{k}\prod_{t=\tau+1}^{k}\mathbf{W}\mathbf{P}_t\mathbf{V}_\tau.
\end{equation}
Substituting Eq.~\ref{eq_36} into Eq.~\ref{eq_37} produces the following relationship:
\begin{equation}
    \begin{split}
        &\mathbf{X}_k=-\alpha\sum_{\tau=1}^{k}\prod_{t=\tau+1}^k\mathbf{W}\mathbf{P}_t\sum_{o=0}^{\tau-1}\beta^{\tau-1-o}\mathbf{G}_o\\&=-\alpha \sum_{\tau=1}^{k}\sum_{o=0}^{\tau-1}\beta^{\tau-1-o}\prod_{t=\tau+1}^k\mathbf{W}\mathbf{P}_t\mathbf{G}_o\\&=-\alpha\sum_{c=0}^{k-1}[\sum_{l=c+1}^k\beta^{l-1-c}]\prod_{t=c+1}^{k-1}\mathbf{W}\mathbf{P}_t\mathbf{G}_c\\&=-\alpha\sum_{\tau=0}^{k-1}\frac{1-\beta^{k-\tau}}{1-\beta}\prod_{t=\tau+1}^{k-1}\mathbf{W}\mathbf{P}_t\mathbf{G}_\tau.
    \end{split}
\end{equation}
Multiplying $\mathbf{I-Q}$ on both sides yields
\begin{equation}
    \mathbf{(I-Q})\mathbf{X}_k=-\alpha\sum_{\tau=0}^{k-1}\frac{1-\beta^{k-\tau}}{1-\beta}\mathbf{(I-Q})\prod_{t=\tau+1}^{k-1}\mathbf{W}\mathbf{P}_t\mathbf{G}_\tau
\end{equation}
It is observed that the above equality has the similar form of Eq.~\ref{eq_14} and we process it as done in Eq.~\ref{eq_16}. With Lemma~\ref{lemma_2} in hand, we present a key lemma for assisting in the proof of Theorem~\ref{dmm-msgd-theo}.
\begin{lemma}\label{lemma_6}
    Let Assumptions~\ref{assumption_1} and~\ref{assumption_3} hold. For $\{\bar{\mathbf{x}}_k\}$ defined in Algorithm~\ref{alg:dmm_msgd}, if $\alpha\leq \frac{(1-\beta)(1-\sqrt{\rho})}{4\sqrt{2}L}$, then for all $K\geq 1$, we have
    \begin{equation}
    \begin{split}
        &\sum_{k=1}^K\frac{1}{N}\sum_{i=1}^N\mathbb{E}[\|\bar{\mathbf{x}}_k-\mathbf{x}^i_k\|^2]\leq \frac{4\alpha^2\sigma^2K}{(1-\beta)^2(1-\rho')}\\&+\frac{16\alpha^2}{(1-\sqrt{\rho'})^2(1-\beta)^2}\sum_{k=1}^K\mathbb{E}[\|\frac{1}{N}\sum_{i=1}^N\nabla f^i(\mathbf{x}^i_k)\|^2]\\&+\frac{16K\alpha^2\kappa^2}{(1-\sqrt{\rho'})^2(1-\beta)^2},
    \end{split}
    \end{equation}
where $\bar{\mathbf{x}}_k=\frac{1}{N}\sum_{i=1}^N\mathbf{x}^i_k$.
\end{lemma}
The proof of Lemma~\ref{lemma_6} follows from the proof of Lemma 1 in~\cite{esfandiari2021cross} and Lemma 11 in~\cite{yu2019linear}. With this in hand, we are now ready to prove Theorem~\ref{dmm-msgd-theo}. The proof techniques are quite similar as used in showing Theorem~\ref{dmm-sgd-theo}. Specifically, we apply the smoothness condition and conclusion from Lemma~\ref{lemma_6} to arrive at the conclusion. The proof also follows similarly from Theorem 1 in~\cite{esfandiari2021cross} and Theorem 3 in~\cite{yu2019linear}. The proof for Corollary~\ref{dmm-msgd-coro} is immediately shown by substituting the step size into the conclusion from Theorem~\ref{dmm-msgd-theo}.
\subsection{Additional Analysis for \textsc{DIMAT-AMSGrad}}
The proof for Theorem~\ref{dmm-adam-theo} is fairly non-trivial and technical. In the proof, we need to use an auxiliary sequence as $\bar{\mathbf{p}}_k$ defined before. Therefore, we utilize the same auxiliary variable in the proof. 
Similarly, we next establish the relationship between $\bar{\mathbf{p}}_k$ and $\bar{\mathbf{p}}_{k+1}$. Please note that in the analysis, we may use different notations specified for the convenience of analysis.
\begin{lemma}
For the sequence defined in Eq.~\ref{auxi_rel}, through Algorithm~\ref{alg:dmm_adam}, we have the following relationship
\begin{equation}
\begin{split}
   \bar{\mathbf{p}}_{k+1}- \bar{\mathbf{p}}_{k}&=\alpha\frac{\beta_1}{1-\beta_1}\frac{1}{N}\mathbf{m}^i_k\oslash((\mathbf{u}^i_{k-1})^{1/2}-(\mathbf{u}^i_k)^{1/2})\\&-\alpha\frac{1}{N}\sum^N_{i=1}\mathbf{g}^i_k\oslash(\mathbf{u}^i_k)^{1/2}
\end{split}
\end{equation}
\end{lemma}
The proof follows similarly from Lemma A.1 in~\cite{chen2023convergence}, 
Due to the $\textnormal{max}(\cdot;\cdot)$ function in the update law, handling such a function can impose difficulties in the proof. Therefore, we present a lemma to pave the way.
\begin{lemma}
    Define a set of numbers, $c_1, c_2, ..., c_n\in\mathbb{R}$ and denote their mean by $\bar{c}=\frac{1}{n}\sum_{i=1}^nc_i$. Define $h_i(r):=\textnormal{max}\{c_i,r\}$ and $\bar{h}(r)=\frac{1}{n}\sum_{i=1}^nh_i(r)$. For any $r$ and $r'$ with $r'\geq r$, we have
    \begin{equation}
        \sum_{i=1}^n|h_i(r)-\bar{h}_i(r)|\leq \sum_{i=1}^n|h_i(r')-\bar{h}_i(r')|,
    \end{equation}
and when $r\leq \textnormal{min}_{i\in[n]}c_i$, we have
\begin{equation}
\sum_{i=1}^n|h_i(r)-\bar{h}_i(r)|=\sum_{i=1}^n|c_i-\bar{c}|.
\end{equation}
\end{lemma}
With the above two lemmas, we are now ready to show the detailed proof for Theorem~\ref{dmm-adam-theo}. Please note that the proof techniques follow from the majority of proof of Theorems 2 and 3 in~\cite{chen2023convergence}. However, the significant difference is to incorporate the permutation matrix $\mathbf{P}$ into the update law such that it leads to the impact of the spectral gap on the error bounds. We will next arrive at this with the derivation. We first define two auxiliary variables:
\[\mathbf{M}_k=[\mathbf{m}^1_k;\mathbf{m}^2_k;...;\mathbf{m}^N_k]^\top\in\mathbb{R}^{dN},\] and \[\mathbf{U}_k=[\mathbf{u}^1_k;\mathbf{u}^2_k;...;\mathbf{u}^N_k]^\top\in\mathbb{R}^{dN}\]
Based on Algorithm~\ref{alg:dmm_adam}, we have
\begin{equation}
    \mathbf{X}_k = \mathbf{W}\mathbf{P}_{k-1}\mathbf{X}_{k-1}-\alpha\mathbf{M}_{k-1}\oslash\mathbf{U}_{k-1}^{1/2}.
\end{equation}
Recursively applying the above equation yields
\begin{equation}
    \mathbf{X}_k =\prod_{\tau=1}^{k-1}\mathbf{W}\mathbf{P}_\tau\mathbf{X}_1-\alpha \sum_{\tau=1}^{k-1}\prod_{t=\tau+1}^{k-1}\mathbf{W}\mathbf{P}_t\mathbf{M}_\tau\oslash \mathbf{U}_\tau^{1/2}
\end{equation}
Setting 0 initial condition and multiplying by $\mathbf{I-Q}$ on both sides attains the following relationship:
\begin{equation}
    (\mathbf{I-Q})\mathbf{X}_k =-\alpha \sum_{\tau=1}^{k-1}(\mathbf{I-Q})\prod_{t=\tau+1}^{k-1}\mathbf{W}\mathbf{P}_t\mathbf{M}_\tau\oslash \mathbf{U}_\tau^{1/2}
\end{equation}
We then calculate its squared norm and take the expectation to get the similar equation as Eq.~\ref{consensus}.
\begin{eqnarray*}
    \mathbb{E}[\|(\mathbf{I-Q})\mathbf{X}_k\|^2]
    =&\alpha^2\mathbb{E}[\|\sum_{\tau=1}^{k-1}(\mathbf{I-Q}) \\
    &\prod_{t=\tau+1}^{k-1}\mathbf{W}\mathbf{P}_t\mathbf{M}_\tau\oslash \mathbf{U}_\tau^{1/2}\|^2].
\end{eqnarray*}
For \textsc{DIMAT-SGD} and \textsc{DIMAT-MSGD}, to acquire the upper bound of $\mathbb{E}[\|(\mathbf{I-Q})\mathbf{X}_k\|^2]$, we used the trick $\mathbf{G}_\tau-\mathbf{H}_\tau+\mathbf{H}_\tau$ as there is no assumption for bounded (stochastic) gradients. However, For Adam type of algorithms, to the best of our knowledge, this assumption is still required to achieve the convergence. Regarding its relaxation we will leave in our future work. Hence, based on Lemma~\ref{lemma_2}, we have the following relationship
\begin{equation}
\mathbb{E}[\|(\mathbf{I-Q})\mathbf{X}_k\|^2]\leq \alpha^2\mathbb{E}[\|\sum_{\tau=1}^{k-1}(\rho')^{k-1-\tau}\mathbf{M}_\tau\oslash\mathbf{U}^{1/2}_\tau\|^2]
\end{equation}
Thus, based on the conditions in Theorem~\ref{dmm-adam-theo}, we can easily get the inequality as follows
\begin{equation}\label{eq_51}
   \mathbb{E}[\|(\mathbf{I-Q})\mathbf{X}_k\|^2]\leq\frac{\alpha^2 NdG^2_\infty}{(1-\rho')^2\epsilon}, 
\end{equation}

which holds due to $\|\mathbf{g}^i_k\|\leq G_\infty, [\mathbf{U}^i_k]_j\geq \epsilon$. Similarly, the upper bound of $\mathbb{E}[\|\bar{\mathbf{p}}_k-\bar{\mathbf{x}}_k\|^2]$ is as follows\begin{equation}\label{eq_52}\begin{split}
\mathbb{E}[\|\bar{\mathbf{p}}_k-\bar{\mathbf{x}}_k\|^2]&= \mathbb{E}[\|\frac{\beta_1}{1-\beta_1}(\bar{\mathbf{x}}_k-\bar{\mathbf{x}}_{k-1})\|^2]\\&=(\frac{\beta_1}{1-\beta_1})^2\alpha^2\mathbb{E}[\|\frac{1}{N}\sum_{i=1}^N\mathbf{m}^i_{k-1}\oslash(\mathbf{u}^i_k)^{1/2}\|^2]\\&\leq (\frac{\beta_1}{1-\beta_1})^2\frac{\alpha^2dG^2_\infty}{N\epsilon}\end{split}\end{equation}
Therefore, we can observe how the permutation matrix can be squeezed in the analysis such that the error bound is impacted with respect to the spectral gap $1-\rho'$. It also implies that existing analysis can be adapted to give the improved error bound shown in Theorem~\ref{dmm-adam-theo}. Thus, we are not going to repeat all proof steps that are similar to existing analysis in~\cite{chen2023convergence}, while, instead, giving the proof sketch, which assists in arriving at Theorem~\ref{dmm-adam-theo}.
\begin{proof}
We now present the proof sketch for Theorem~\ref{dmm-adam-theo} and will refer interested readers to related works for more detail.
\begin{itemize}
    \item \textit{Step 1: Bounding gradient.} With the assistance of auxiliary sequence $\{\bar{\mathbf{p}}_k\}$, we don't have to consider the complicated update dependence on $\mathbf{m}_k$ and thus perform convergence analysis for the upper bound on $\nabla f(\bar{\mathbf{p}}_k)$. With this in hand, based on the smoothness of $f$, we subsequently construct the bound for $\frac{1}{K}\sum_{k=1}^K\mathbb{E}[\|\frac{\nabla f(\bar{\mathbf{x}}_k)}{(\bar{\mathbf{u}}_k)^{1/4}}\|^2]$, where $\bar{\mathbf{u}}_k=\frac{1}{N}\sum_{i=1}^N\mathbf{u}^i_k$.
    \begin{equation}
    \begin{split}
        &\frac{1}{K}\sum_{k=1}^K\mathbb{E}[\|\frac{\nabla f(\bar{\mathbf{x}}_k)}{(\bar{\mathbf{u}}_k)^{1/4}}\|^2]\leq \frac{2\mathbb{E}[f(\bar{\mathbf{p}}_1)-f(\bar{\mathbf{p}}_{K+1})]}{K\alpha}\\&+\frac{2\beta_1D_1}{K(1-\beta_1)}+\frac{2D_2}{K}+\frac{3D_3}{K}+\frac{L\mathbb{E}[\|\bar{\mathbf{p}}_{k+1}-\bar{\mathbf{p}}_k\|^2]}{K\alpha},
    \end{split}
    \end{equation}
    where \[\begin{split}D_1 &= \sum_{k=1}^K\mathbb{E}[\langle\nabla f(\bar{\mathbf{p}}_k),\frac{1}{N}\sum_{i=1}^N\mathbf{m}^i_{k-1}\oslash ((\mathbf{u}^i_{k-1})^{1/2}\\&-(\mathbf{u}^i_k)^{1/2})\rangle],\end{split}\]\[\begin{split}D_2 &= \sum_{k=1}^K\mathbb{E}[\langle\nabla f(\bar{\mathbf{p}}_k),\frac{1}{N}\sum_{i=1}^N\nabla f^i(\mathbf{x}^i_k)\oslash ((\bar{\mathbf{u}}_k)^{1/2}\\&-(\mathbf{u}^i_k)^{1/2})\rangle],\end{split}\]\[\begin{split}D_3 &= \sum_{k=1}^K\mathbb{E}[\|\frac{\frac{1}{N}\sum_{i=1}^N\nabla f^i(\mathbf{x}^i_k)-\nabla f(\bar{\mathbf{x}}_k)}{(\bar{\mathbf{u}}_k)^{1/4}}\|^2\\&+\|\frac{\nabla f(\bar{\mathbf{p}}_k)-\nabla f(\bar{\mathbf{x}}_k)}{(\bar{\mathbf{u}}_k)^{1/4}}\|^2].\end{split}\]
    The smoothness condition and Eqs.~\ref{eq_51} and~\ref{eq_52} grant us the upper bound of $D_3$. Establishing the upper bounds for $D_1$ and $D_2$ give rise to the terms related to $\mathbb{E}[\sum_{k=1}^K\|\mathbf{V}_{k-1}-\mathbf{V}_{k-2}\|_{abs}]$, where $\|\mathbf{C}\|_{abs}=\sum_{i,j}|\mathbf{C}_{i,j}|$ denotes the entry-wise $L_1$ norm of a matrix. $\mathbf{V}_k$ is established as a non-decreasing function such that $\mathbb{E}[\sum_{k=1}^K\|\mathbf{V}_{k-1}-\mathbf{V}_{k-2}\|_{abs}]=\mathbb{E}[\sum_{i=1}^N\sum_{j=1}^d([\mathbf{v}^i_{K-1}]_j-[\mathbf{v}^i_0]_j)]$. Due to $\|\mathbf{g}^i_k\|_\infty\leq G_\infty$, it is proved that $[\mathbf{v}^i_k]_j\leq G^2_\infty$. With this, we can conclude that $\mathbb{E}[\sum_{k=1}^K\|\mathbf{V}_{k-1}-\mathbf{V}_{k-2}\|_{abs}]\leq NdG^2_\infty$.
    \item \textit{Step 2: Bounding the drift term variance.} One important term in the proof is the stochastic gradient variance multiplied by the adaptive learning rate, $\mathbb{E}[\|\frac{1}{N}\sum_{i=1}^N\mathbf{g}^i_k\oslash\mathbf{u}^i_k\|^2]\leq\mathbb{E}[\|\frac{1}{N}\sum_{i=1}^N\nabla f^i(\mathbf{x}^i_k)\oslash(\mathbf{u}^i_k)^{1/2}\|^2]+\frac{d\sigma^2}{N\epsilon}$. To process the first term on the right side of the above inequality, we can use $\bar{\mathbf{u}}_k$ and $\|\mathbf{a}+\mathbf{b}\|^2\leq 2\|\mathbf{a}\|^2+2\|\mathbf{b}\|^2$ to transform from $\mathbb{E}[\|\frac{1}{N}\sum_{i=1}^N\nabla f^i(\mathbf{x}^i_k)\oslash(\mathbf{u}^i_k)^{1/2}\|^2]$ to $\mathbb{E}[\|\frac{1}{N}\sum_{i=1}^N\nabla f^i(\mathbf{x}^i_k)\oslash(\bar{\mathbf{u}}_k)^{1/2}\|^2]$. We then can bound them as performed for $D_2$ and $D_3$. Hence, we will reach to the bound in the following for $\frac{1}{K}\sum_{k=1}^K\mathbb{E}[\|\frac{\nabla f(\bar{\mathbf{x}}_k)}{(\bar{\mathbf{u}}_k)^{1/4}}\|^2]$:
    \begin{equation}\label{eq_53}
    \begin{split}
        &\frac{1}{K}\sum_{k=1}^K\mathbb{E}[\|\frac{\nabla f(\bar{\mathbf{x}}_k)}{(\bar{\mathbf{u}}_k)^{1/4}}\|^2]\leq C_1(\frac{\mathbb{E}[f(\bar{\mathbf{p}}_1)-f^*]}{K\alpha}+\frac{\alpha d\sigma^2}{N})\\&+C_2\alpha^2d+C_3\alpha^3d+\frac{C_4+C_5\alpha}{K\sqrt{N}}NdG^2_\infty,
    \end{split}
    \end{equation}
    where $C_1=\textnormal{max}\{4, 4L/\epsilon\}$, $C_2=6(\frac{\beta^2_1}{(1-\beta_1)^2}+\frac{1}{(1-\rho')^2})\frac{LG^2_\infty}{\epsilon^{1.5}}$, $C_3=\frac{16L^2G^2_\infty}{\epsilon^2}$, $C_4=\frac{2(\lambda+\frac{\beta_1}{1-\beta_1})G^2_\infty}{\epsilon^{1.5}(1-\lambda)}$, $C_5=\frac{2LG^2_\infty(\lambda+\frac{\beta_1}{1-\beta_1}+2)}{\epsilon^2(1-\lambda)}$.
    \item \textit{Setting explicit step size.} Setting the step size as $\alpha=\mathcal{O}(\sqrt{\frac{1}{dK}})$, substituting it into Eq.~\ref{eq_53}, and using the fact that $\|\bar{\mathbf{u}}_k\|_\infty\leq G^2_\infty$ yields the desired result.
\end{itemize}
With the above three steps, the desired conclusion is obtained.
\end{proof}
\begin{figure*}[ht!]    
    \begin{subfigure}[b]{0.49\textwidth}
        \centering
        \includegraphics[width=\linewidth]{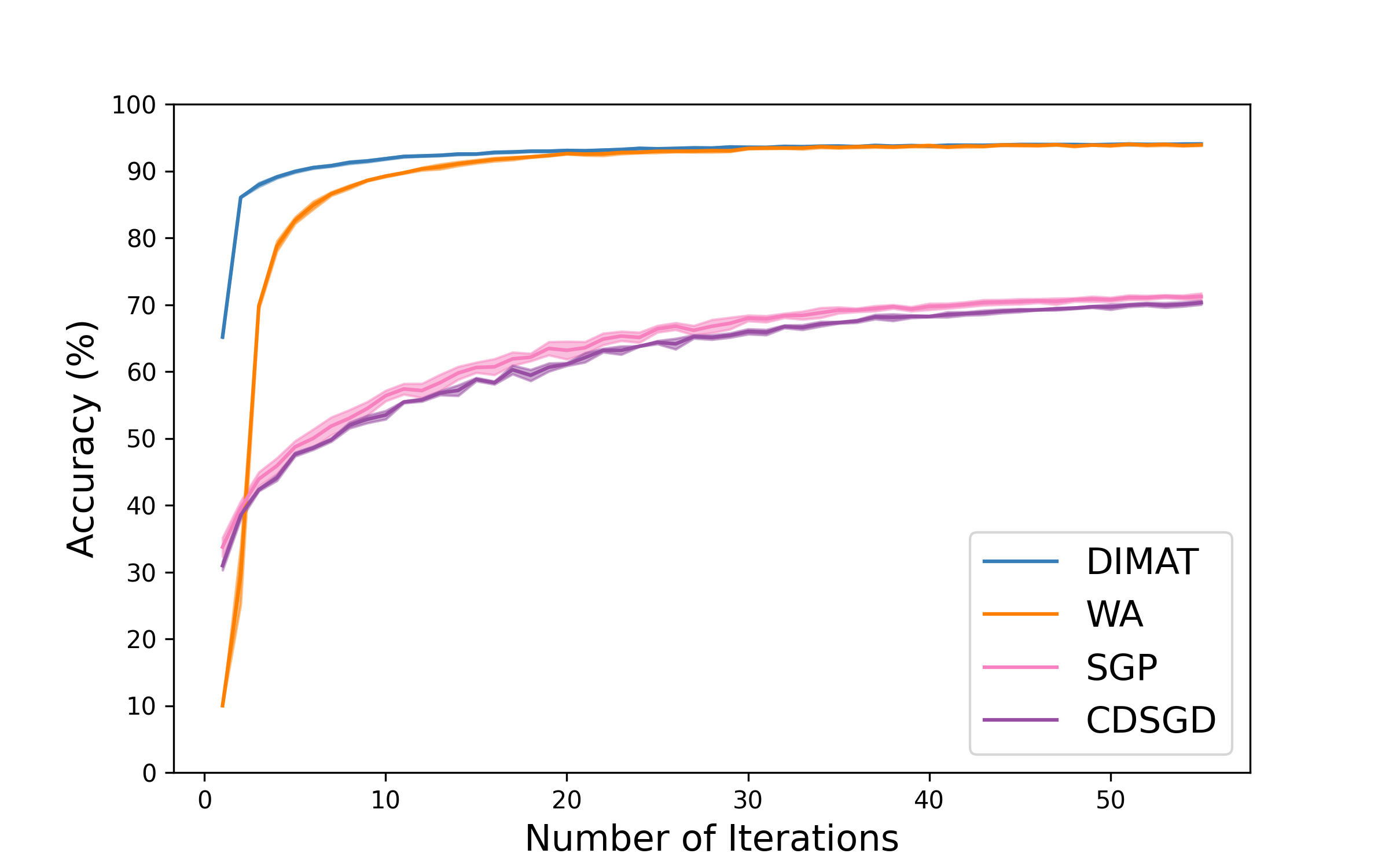}
        \caption{Fully Connected}
        \label{fig:cifar10FC}
    \end{subfigure}
    \begin{subfigure}[b]{0.49\textwidth}
        \centering
        \includegraphics[width=\linewidth]{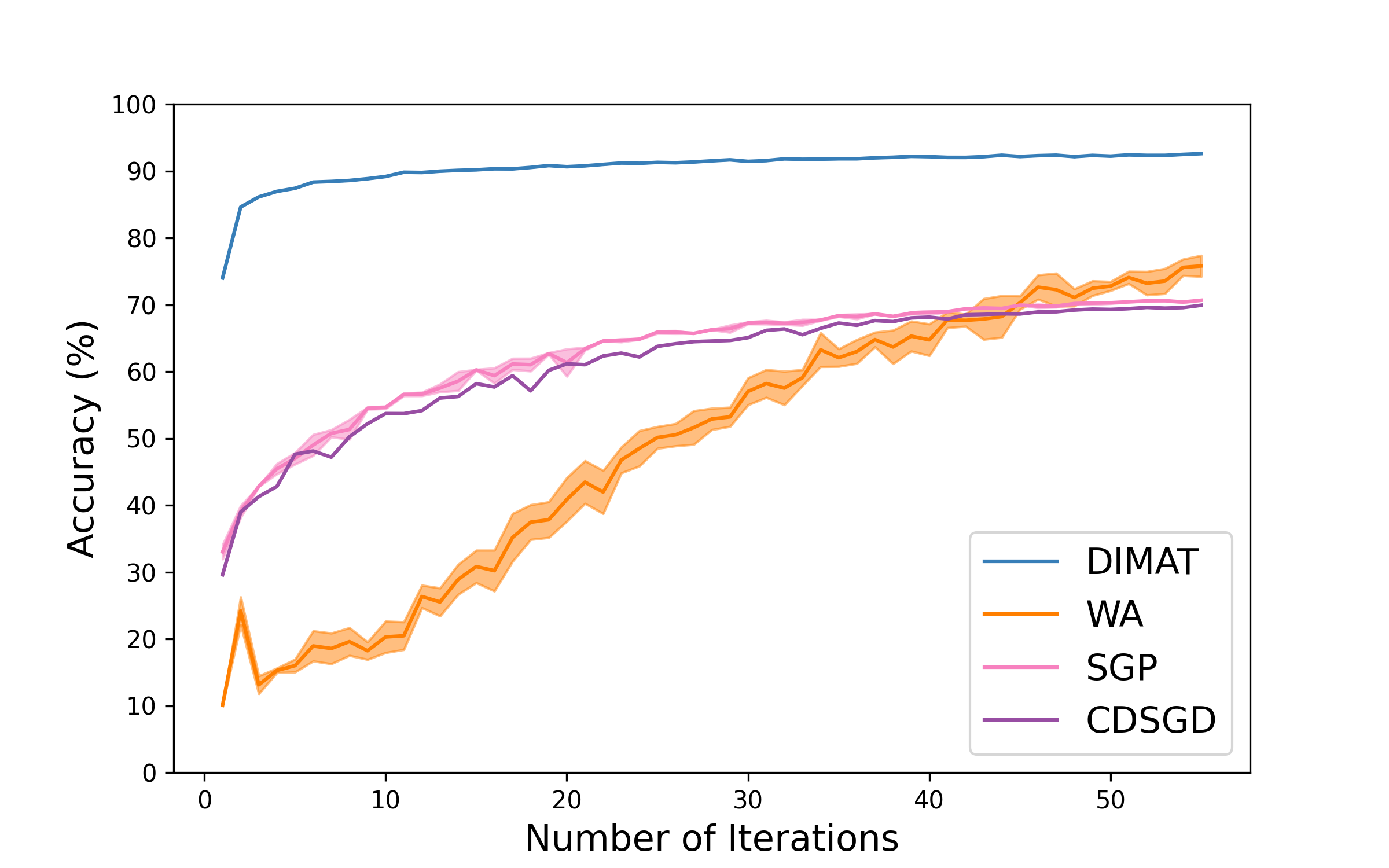}
        \caption{Ring}
        \label{fig:cifar10R}
    \end{subfigure}
    \caption{Comparing algorithmic accuracy (mean$\pm$std) in fully connected (a) and ring (b) topologies with ResNet-20 architecture on CIFAR-10 IID data for 5 agents.}
    \label{fig:iid_cifar10}
\end{figure*}
\begin{figure*}[h!]    
    \begin{subfigure}[b]{0.49\textwidth}
        \centering
        \includegraphics[width=\linewidth]{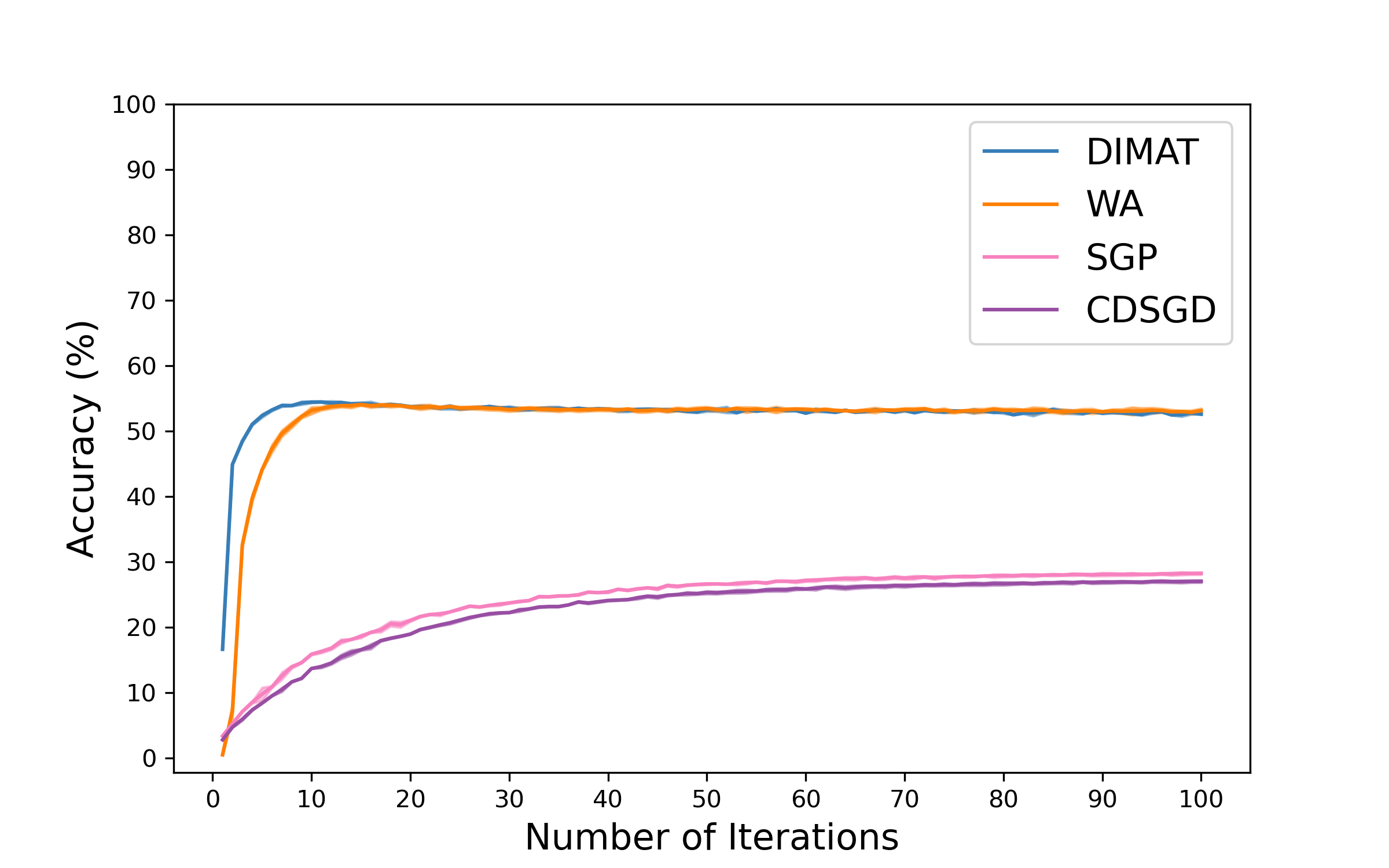}
        \caption{Fully Connected}
        \label{fig:tinyimFC}
    \end{subfigure}
    \begin{subfigure}[b]{0.49\textwidth}
        \centering
        \includegraphics[width=\linewidth]{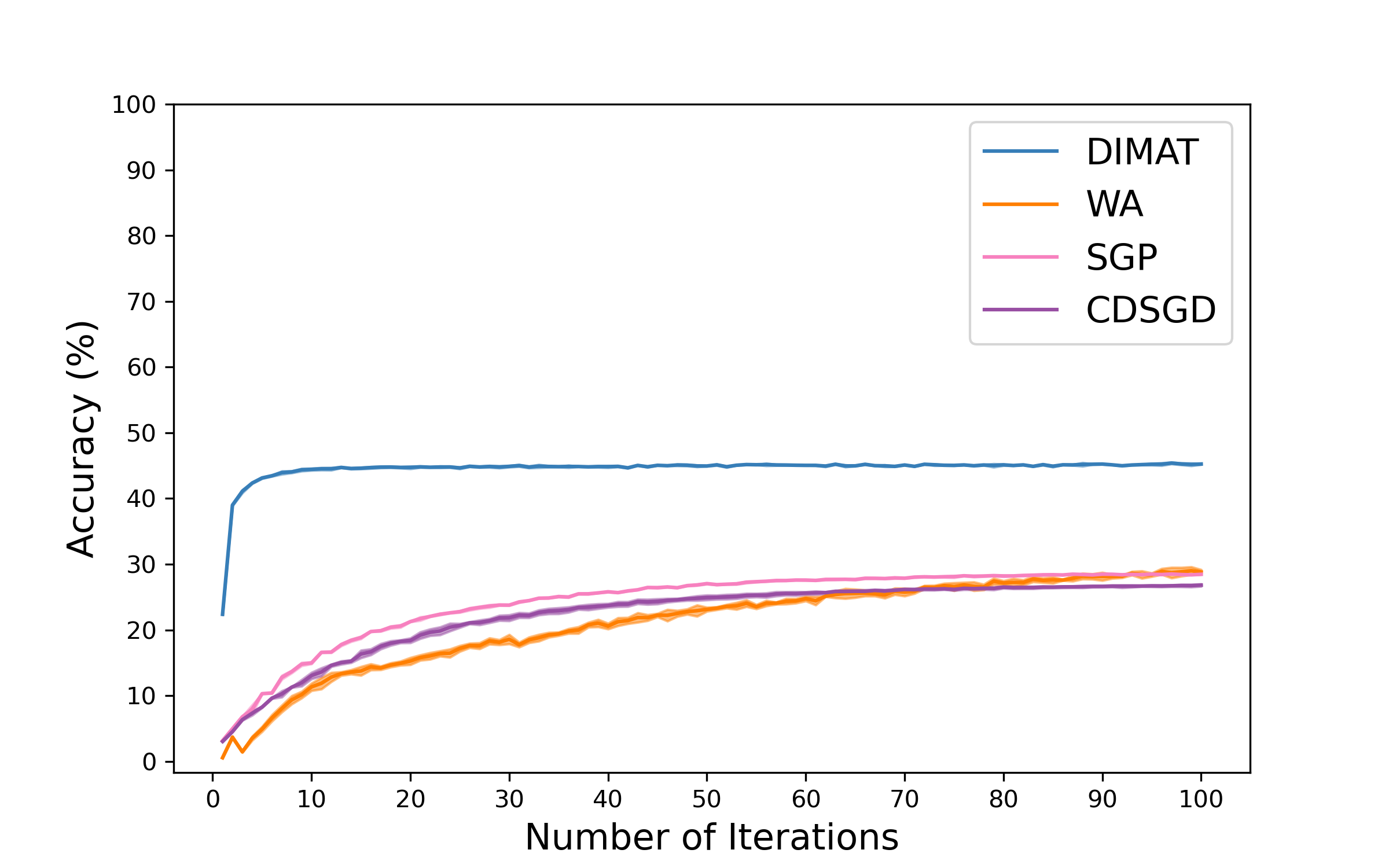}
        \caption{Ring}
        \label{fig:tinyimR}
    \end{subfigure}
    \caption{Comparing algorithmic accuracy (mean$\pm$std) in fully connected (a) and ring (b) topologies with ResNet-20 architecture on Tiny ImageNet IID data for 5 agents.}
    \label{fig:iid_tinyim}
\end{figure*}
\begin{figure*}[h!]    
    \begin{subfigure}[b]{0.49\textwidth}
        \centering
        \includegraphics[width=\linewidth]{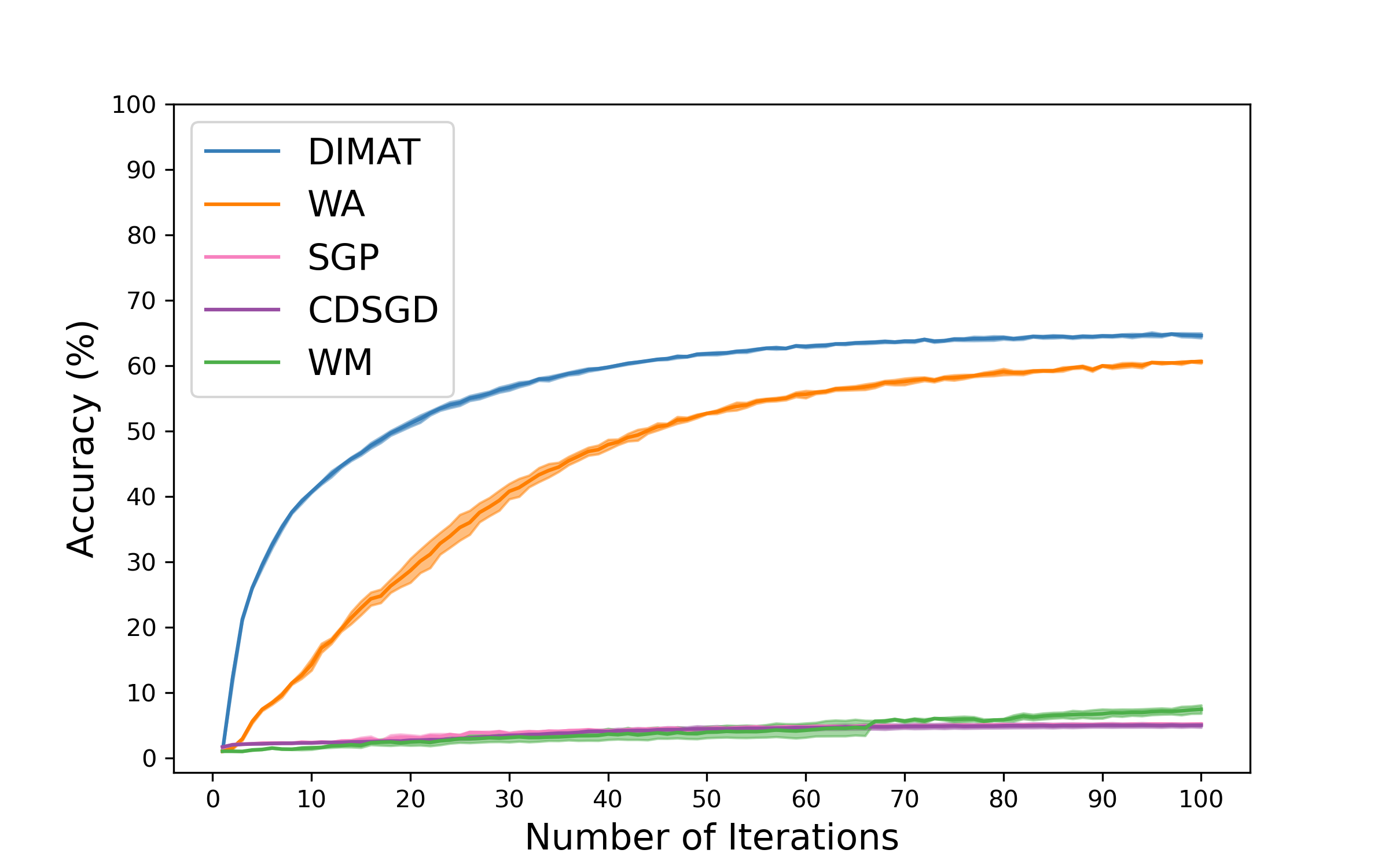}
        \caption{10 Agents}
        \label{fig:10vgg16}
    \end{subfigure}
    \begin{subfigure}[b]{0.49\textwidth}
        \centering
        \includegraphics[width=\linewidth]{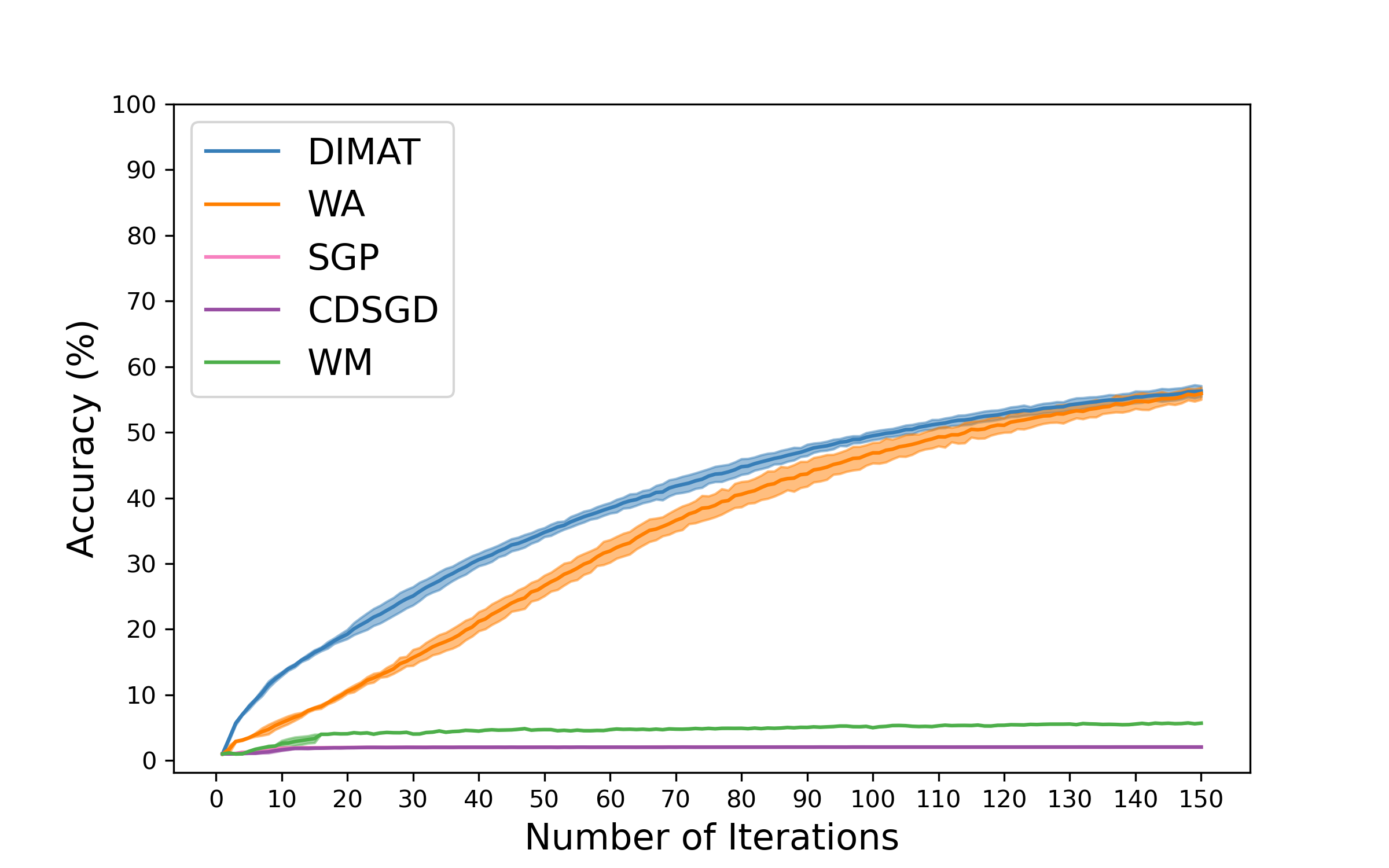}
        \caption{20 Agents}
        \label{fig:20vgg16}
    \end{subfigure}
    \caption{Comparing algorithmic accuracy (mean$\pm$std) for ten (a) and twenty (b) agents with VGG16 architecture on CIFAR-100 IID data.}
    \label{fig:iid_scale}
\end{figure*}
\section{Additional Experimental Results}\label{sec:additional_results}
Within this section, we present supplementary experimental results encompassing additional datasets, specifically Tiny ImageNet and CIFAR10, along with an alternative architecture, ResNet50.

Moreover, our investigation delves into scalability, considering 10 and 20 models across various baseline algorithms. Additionally, we scrutinize the impact of different epoch configurations for DIMAT, providing insights into its performance over successive iterations.

Furthermore, we sought to explore the key factors contributing to the challenge of scalability in non-iid scenarios. By employing random initialization techniques and incorporating larger datasets, we aimed to assess the efficacy of our model under diverse conditions.

\subsection{Additional Dataset Comparisons}\label{subsec:additional_dataset}
In this subsection, we present an analysis of the Tiny ImageNet and CIFAR10 datasets. Table \ref{tab:additional_datasets} provides a comprehensive overview of algorithmic performance across these additional datasets. The evaluation considers two key scenarios for each dataset: Fully Connected (FC) and Ring topologies in non-IID data. The reported values represent the mean and standard deviation of the performance metric obtained through multiple trial runs.

\begin{table}[h!]
\caption{Comparing algorithmic accuracy (mean$\pm$std) in fully connected and ring topologies with ResNet-20 architecture on Tiny ImageNet and CIFAR-10 non-IID data for 5 agents.}
\label{tab:additional_datasets}
\resizebox{\columnwidth}{!}{%
\begin{tabular}{lllll}
\hline
\multirow{2}{*}{\textbf{Algorithm}} & \multicolumn{2}{c}{\textbf{Tiny ImageNet}}        & \multicolumn{2}{c}{\textbf{CIFAR10}}              \\ \cline{2-5} 
                                    & FC                      & Ring                    & FC                      & Ring                    \\ \hline
SGP                                 & 9.49$\pm$0.43           & 7.42$\pm$0.24           & 19.18$\pm$0.11          & 19.04$\pm$0.27          \\
CDSGD                               & 9.05$\pm$0.15           & 7.36$\pm$0.25           & 18.85$\pm$0.08          & 19.20$\pm$0.16          \\
WA                                  & 48.59$\pm$0.71          & 10.48$\pm$0.25          & \textbf{49.25$\pm$3.95} & \textbf{23.14$\pm$1.46} \\
DIMAT (ours)                        & \textbf{49.09$\pm$0.23} & \textbf{17.70$\pm$0.14} & 27.12$\pm$3.39          & 20.22$\pm$0.20          \\ \hline
\end{tabular}%
}
\end{table}

Notably, in the case of Tiny ImageNet, our proposed algorithm, DIMAT, emerges as a standout performer, outperforming all baseline algorithms across both non-IID and IID scenarios, as it can be seen in Table \ref{tab:additional_datasets} and fig.\ref{fig:iid_tinyim}. However, for CIFAR10, the limited pretraining on only two classes results in a bias among agents, impacting their learning effectiveness. This bias is evident in their suboptimal performance. Nevertheless, in an IID setting, as depicted in fig. \ref{fig:iid_cifar10}, DIMAT demonstrates superior performance on both fully connected and ring topologies.

\subsection{Additional Architecture Comparisons}\label{subsec:additional_architecture}
The results for ResNet50 are presented in Table \ref{tab:accuracy-comparison-RN50}. Notably, the performances of WA and DIMAT are comparable. It is evident that DIMAT's performance is highly dependent on the chosen architecture. Additionally, in the IID ring scenario, DIMAT emerges as the top performer, surpassing all other algorithms in terms of accuracy.

\begin{table}[ht!]
    \caption{Comparison of Test Accuracy (mean$\pm$std) on CIFAR-100 with ResNet-50 architecture for 5 agents under both IID and non-IID data distribution, considering fully connected (FC) and ring topologies.}
    \label{tab:accuracy-comparison-RN50}
    \resizebox{\columnwidth}{!}{%
    \begin{tabular}{lllll}
    \hline
    \multirow{2}{*}{\textbf{Algorithm}} & \multicolumn{2}{c}{\textbf{IID}}                  & \multicolumn{2}{c}{\textbf{non-IID}}              \\ \cline{2-5} 
                                        & FC                      & Ring                    & FC                      & Ring                    \\ \hline
    SGP                                 & 39.99$\pm$0.45          & 39.74$\pm$0.04          & 13.27$\pm$0.09          & 13.33$\pm$0.12          \\
    CDSGD                               & 37.99$\pm$0.16          & 37.94$\pm$0.38          & 12.90$\pm$0.10          & 9.10$\pm$5.73           \\
    WA                                  & \textbf{49.07$\pm$0.16}          & 32.47$\pm$0.24          & \textbf{46.76$\pm$0.57}          & \textbf{24.21$\pm$0.39}          \\
    DIMAT (ours)                        & 42.59$\pm$1.00          & \textbf{42.06$\pm$0.13}          & 45.10$\pm$0.49          & 19.54$\pm$0.05          \\ \hline
    \end{tabular}%
    }
\end{table}

\subsection{Additional Scalability Analysis}\label{subsec:additional_scalability}
\subsubsection{IID Data Scalability}

Figures \ref{fig:10vgg16} and \ref{fig:20vgg16} depict algorithmic accuracy trends with varying agent numbers using CIFAR-100 IID data and the VGG16 architecture.
Figure \ref{fig:10vgg16} illustrates accuracy trends for 10 agents, providing a snapshot of algorithmic behavior in a moderately scaled scenario. In fig. \ref{fig:20vgg16}, the analysis extends to 20 agents, offering valuable insights into the algorithm's robustness and scalability as agent numbers increase in the IID scenario.
It's noteworthy that, even with this escalation in the number of agents, DIMAT consistently outperforms all baseline algorithms, showcasing its resilience and superior performance in the face of increased scalability. 
\subsubsection{Non-IID Data Scalability}
In investigating scalability under non-IID scenarios, we hypothesized that the bias introduced by pretraining models might be a contributing factor. Given that the number of classes in non-IID settings decreases, pretraining could potentially lead to biased initializations. To test this hypothesis, we experimented with random initialization instead of starting with pretrained models. However, our results indicate that this change in initialization strategy does not significantly affect the performance of DIMAT. The observed trends remain consistent with those obtained when using pretrained models, suggesting that factors beyond initialization bias do not influence DIMAT's performance in non-IID scenarios.
This can be seen in fig. \ref{fig:random_init}.
\begin{figure}[h!]
    \centering
    \includegraphics[width=\linewidth]{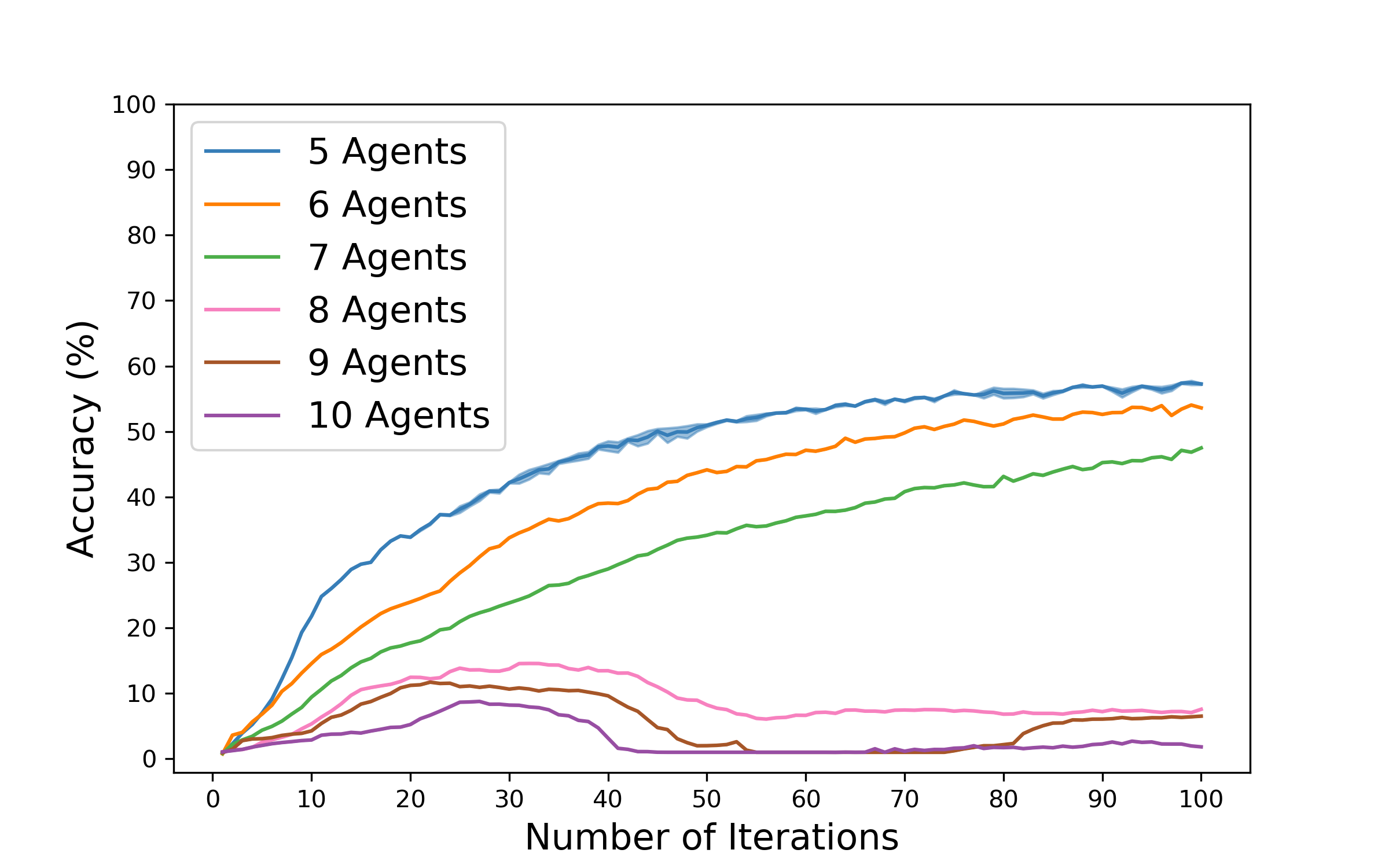}
    \caption{Impact of model's initialization on accuracy (Mean$\pm$Std) using ResNet-20 architecture and a fully connected topology on CIFAR-100 non-IID data. Results show the performance of the DIMAT algorithm with 5 to 10 agents.}
    \label{fig:random_init}
\end{figure}
Another approach we explored to address this issue was utilizing larger datasets, such as Tiny ImageNet, for more agents in non-IID scenarios. As depicted in fig. \ref{fig:large_data}, even with 10 agents, the accuracy continues to increase over time. This suggests that the underlying scalability issue in non-IID scenarios might be influenced by the dataset size. A larger dataset appears to improve performance, particularly for a higher number of agents. However, there seems to be a breakpoint. Unlike CIFAR-100, Tiny ImageNet is capable of handling a larger number of agents but still exhibits a breakpoint as we increase the number of agents, indicating that learning becomes progressively more challenging.

\begin{figure}[h!]
    \centering
    \includegraphics[width=\linewidth]{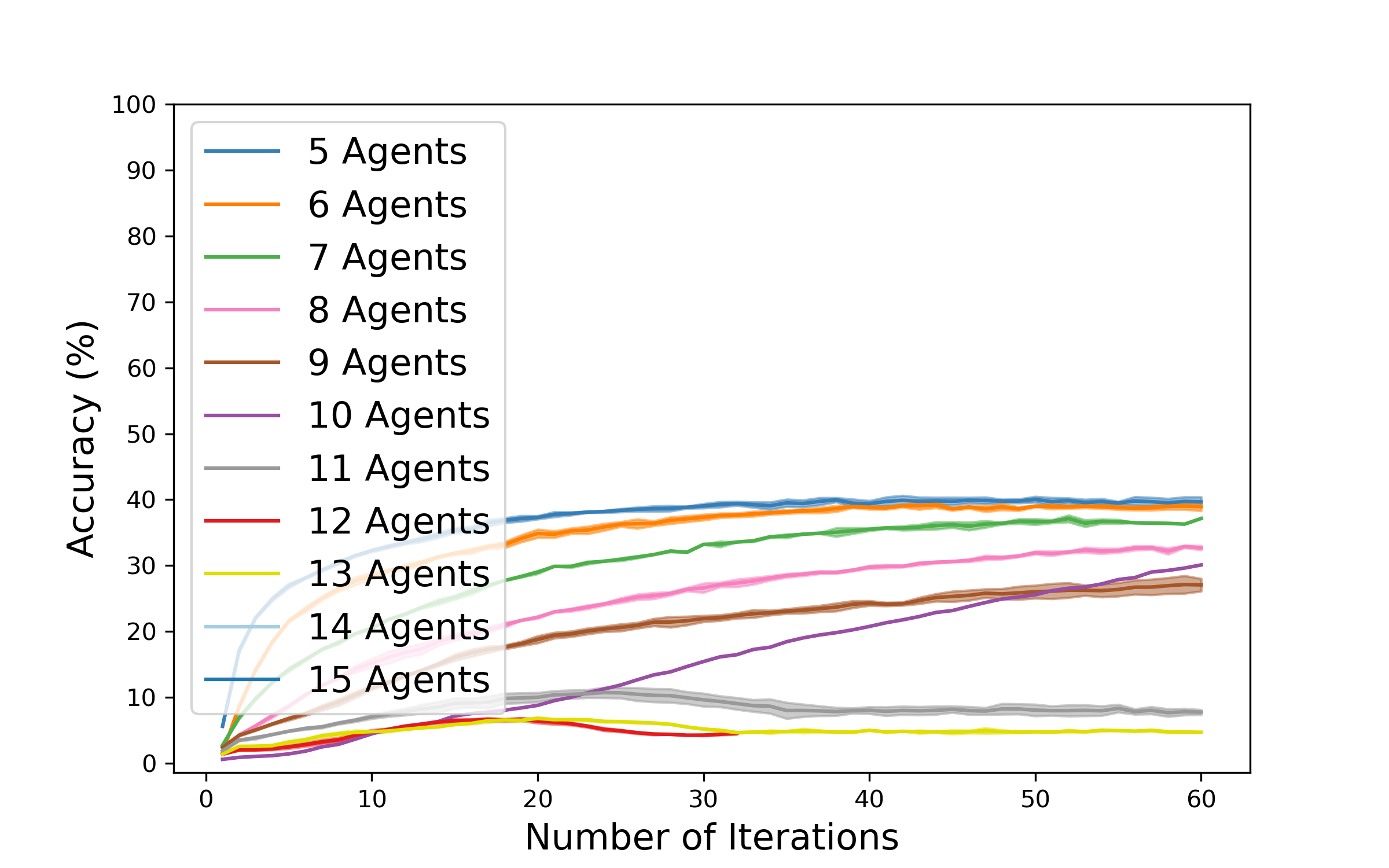}
    \caption{Effect of dataset size on accuracy (Mean$\pm$Std) using ResNet-20 architecture with a fully connected topology on Tiny ImageNet non-IID data. The figure illustrates the performance of the DIMAT algorithm with 5 to 15 agents.}
    \label{fig:large_data}
\end{figure}
\subsection{Exploring Varied Training Epochs}\label{subsec:additional_epochs}
In this subsection, we present results from experiments conducted with different numbers of training epochs for each iteration. Our analysis reveals that the optimal number of training epochs between iterations is 2, outperforming configurations with 1, 5, 7, and 10 training epochs. These findings are illustrated in Fig.~\ref{fig:diff_epochs}.
\begin{figure}[h!]
    \centering
    \includegraphics[width=\linewidth]{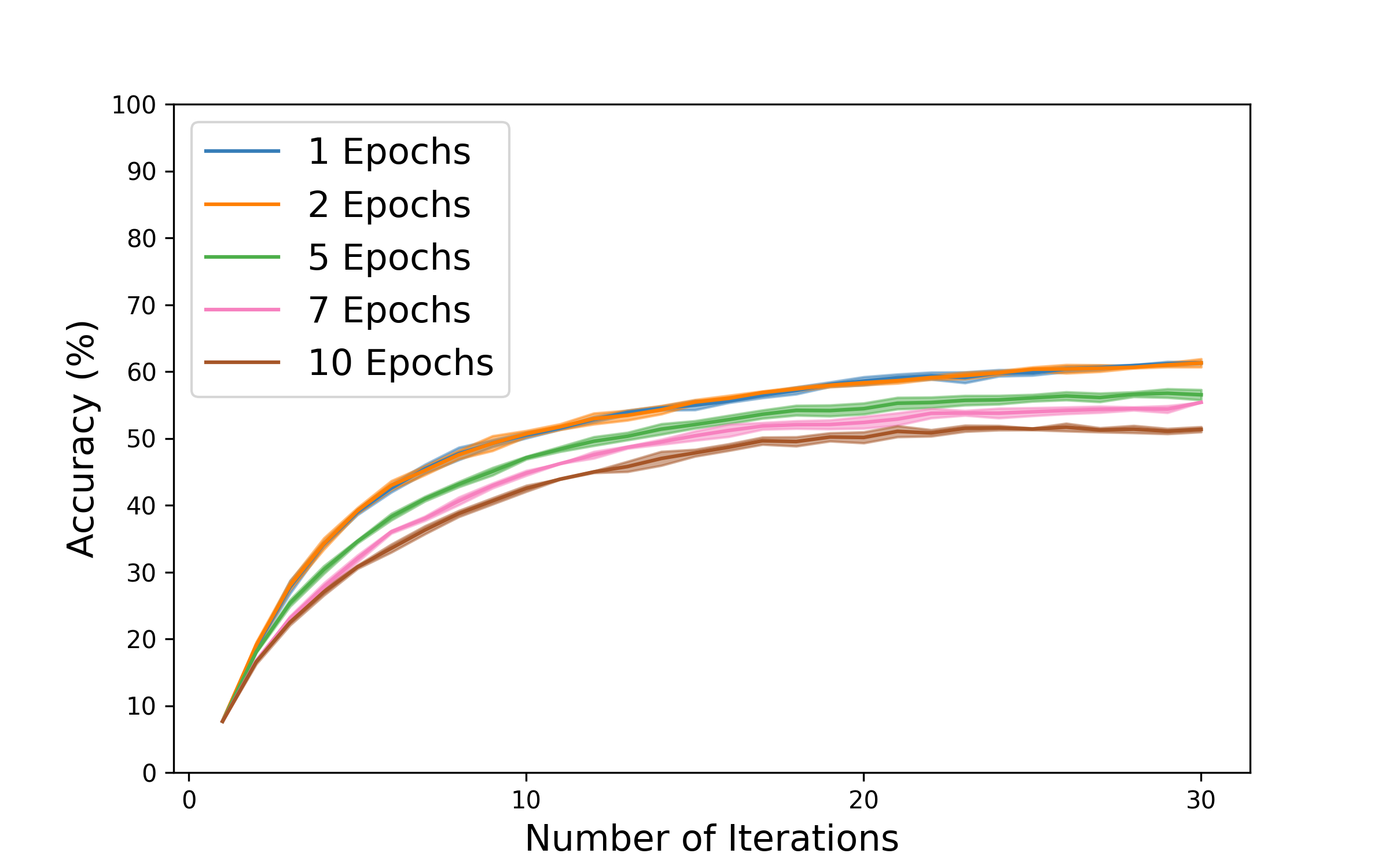}
    \caption{Impact of diverse training epochs on agents accuracy (Mean$\pm$Std) with fully Connected Topology using ResNet-20 architecture on CIFAR-100 non-IID data for 5 agents on the DIMAT algorithm.
.}
    \label{fig:diff_epochs}
\end{figure}
\subsection{Visualization of Communication Overhead}
Figure \ref{fig:comm_overhead} illustrates a comparison of communication overhead among DIMAT, SGP, CGA, and CDSGD for 5, 10, and 20 agents across fully connected and ring topologies.

\begin{figure}[h!]
\centering
\includegraphics[width=\linewidth]{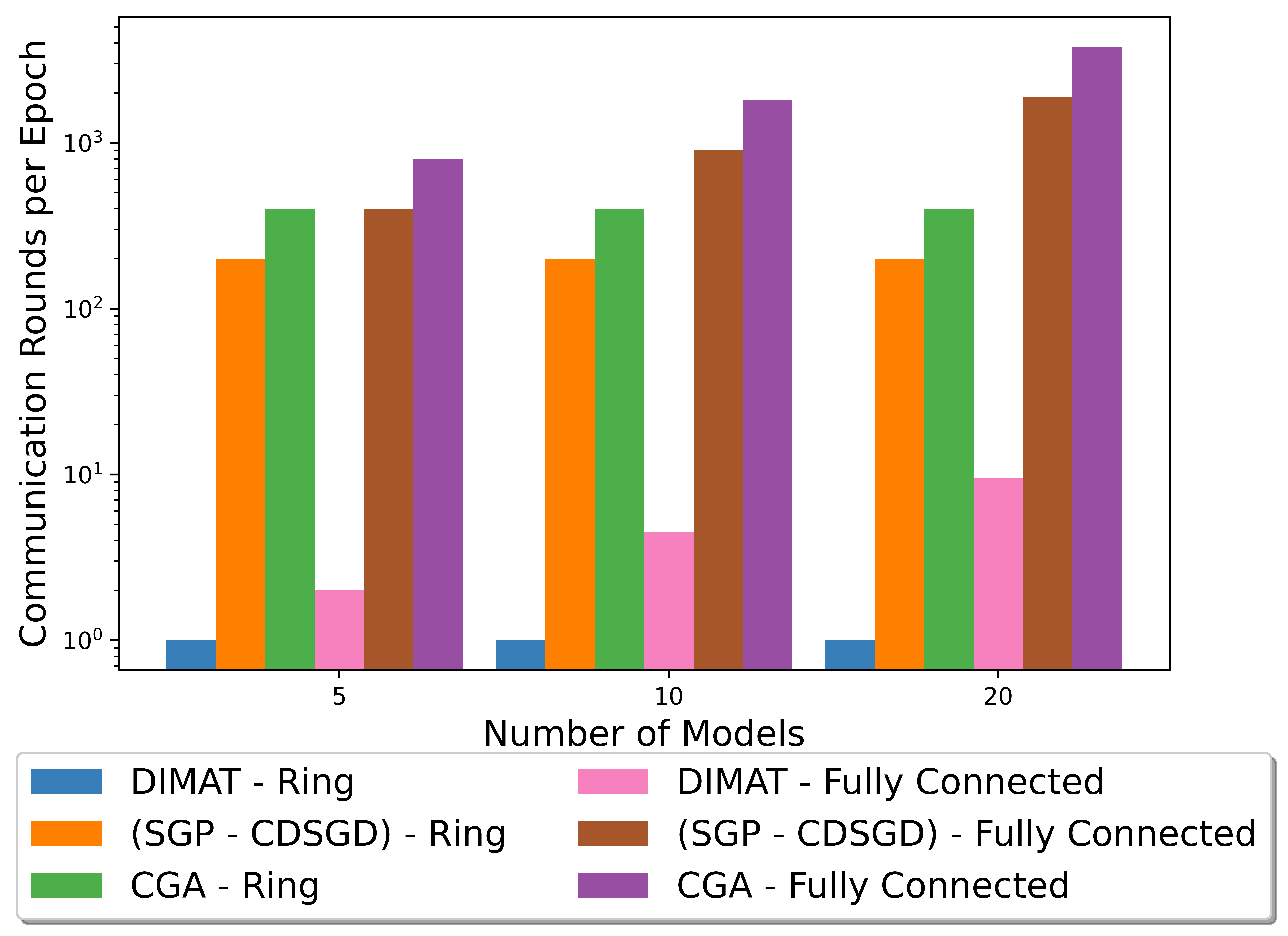}
\caption{Number of communication rounds per epoch for fully connected and ring topologies.}
\label{fig:comm_overhead}
\end{figure}

DIMAT significantly requires fewer communication rounds compared to SGP, CGA, and CDSGD.

\subsection{Computational Overhead}

In this subsection, we examine the computational overhead of various algorithms when training 5 agents using ResNet20 architecture on the non-IID CIFAR-100 dataset for 100 iterations. We focus on GPU memory usage and computation time as key performance metrics. Table \ref{tab:computation_overhead} compares GPU memory usage and computation time for SGP, CDSGD, and our proposed method, DIMAT.

Experiments were conducted on an NVIDIA A100 GPU (80 GB). It is important to note that the reported GPU memory usage is approximate. DIMAT demonstrates significantly lower GPU memory usage, requiring only 6 GB compared to 15 GB for both SGP and CDSGD. Furthermore, DIMAT achieves faster computation, completing the task in 15.95 hours compared to 16.88 hours for SGP and 16.99 hours for CDSGD.

These findings underscore the efficiency of DIMAT in terms of memory usage and computation time, rendering it a promising approach for decentralized learning tasks.

\begin{table}[h]
\centering
\caption{Comparison of GPU memory usage and computation time for 5 agents using ResNet20 on non-IID CIFAR-100 data for 100 iterations. The experiments were conducted on an NVIDIA A100 GPU.}
\label{tab:computation_overhead}
\resizebox{0.7\columnwidth}{!}{%
\begin{tabular}{lll}
\hline
Algorithm    & GPU           & Time                \\ \hline
\small SGP          & \small 15 GB         & \small 16.88 hrs.          \\
\small CDSGD        & \small 15 GB         & \small 16.99 hrs.          \\ 
\small DIMAT (ours) & \small \textbf{6 GB} & \small \textbf{15.95 hrs.} \\ \hline
\end{tabular}%
}
\end{table}

\section{Expanded Explanations of Selected Terminologies}
\subsection{Mixing Matrix}
The mixing matrix, a doubly stochastic matrix, signifies inter-agent influences in collaborative learning systems. While various design choices exist, we adopt a vanilla version for illustration. In a fully connected topology, the matrix is uniform: for instance, in a 5-agent network, all elements are set to 0.2 for symmetrical collaboration. In a ring topology, where agents equally influence their two adjacent counterparts, the matrix takes a circular pattern. Specifically, elements corresponding to the three neighboring agents are 0.333, while the rest are 0. This matrix representation is as follows:

For fully connected topology:
{
\[
\begin{bmatrix}
0.2 & 0.2 & 0.2 & 0.2 & 0.2 \\
0.2 & 0.2 & 0.2 & 0.2 & 0.2 \\
0.2 & 0.2 & 0.2 & 0.2 & 0.2 \\
0.2 & 0.2 & 0.2 & 0.2 & 0.2 \\
0.2 & 0.2 & 0.2 & 0.2 & 0.2 \\
\end{bmatrix}
\]
}

For ring topology:
{
\[
\begin{bmatrix}
0.333 & 0.333 & 0 & 0 & 0.333 \\
0.333 & 0.333 & 0.333 & 0 & 0 \\
0 & 0.333 & 0.333 & 0.333 & 0 \\
0 & 0 & 0.333 & 0.333 & 0.333 \\
0.333 & 0 & 0 & 0.333 & 0.333 \\
\end{bmatrix}
\]
}
\subsection{Activation Matching}
 We adopt the method proposed by Ainsworth et al.~\cite{ainsworth2022git}. This method aims to associate units across two models by performing regression between their activations, under the premise that models must learn similar features to effectively perform the same task.

Given the activations of each model, the objective is to link corresponding units between model 1 ($M_1$) and model 2 ($M_2$), assuming a potential linear relationship between their activations. For activations of the \(\ell\)th layer, represented by $\mathbf{Z}^{(M_1)}$ and $\mathbf{Z}^{(M_2)}$, the goal is to minimize the discrepancy between their activations using a linear assignment problem (LAP), for which efficient algorithms exist.

After solving the assignment problem for each layer, the weights of model 2 are adjusted to closely match those of model 1. This adjustment involves permuting both weights and biases for each layer, resulting in weights that generate activations closely aligned with those of model 1.

This method is computationally efficient, requiring only a single pass over the training dataset to compute activation matrices. Furthermore, activation matching at each layer operates independently of other layers, simplifying the optimization process.